\documentclass{article}

\usepackage[preprint]{neurips_2025}

\usepackage[utf8]{inputenc} 
\usepackage[T1]{fontenc}    
\usepackage{hyperref}       
\usepackage{url}            
\usepackage{booktabs}       
\usepackage{amsfonts}       
\usepackage{nicefrac}       
\usepackage{microtype}      
\usepackage[dvipsnames,table]{xcolor}         

\setcitestyle{numbers,square}

\usepackage{algorithm}
\usepackage{algorithmic}
\usepackage{amsmath}
\usepackage{amssymb}
\usepackage{amsthm}
\usepackage{mathtools}
\usepackage{cleveref}
\usepackage{enumitem}
\usepackage{caption}
\usepackage{subcaption}
\usepackage{array}
\usepackage{graphics}
\usepackage{tikz}
\usetikzlibrary{arrows.meta,positioning,calc,trees,shapes.geometric,backgrounds,matrix}
\usepackage{microtype}
\usepackage{wrapfig}
\usepackage{scalefnt}
\usepackage{sidecap}
\usepackage[bottom]{footmisc}
\usepackage{tcolorbox}
\usepackage[export]{adjustbox}
\usepackage{placeins}

\tcbuselibrary{skins,breakable,hooks}
\usepackage{titlesec}
\titlespacing*{\paragraph}{0pt}{0pt}{1em}

\setlist[itemize]{nosep,   
                  topsep=0pt,
                  partopsep=0pt,
                  leftmargin=1.8em}

\newtheorem{proposition}{Proposition}
\DeclareMathOperator*{\argmax}{argmax}

\newcommand{\bx}{\mathbf{x}}
\newcommand{\bs}{\mathbf{s}}
\newcommand{\by}{\mathbf{y}}
\newcommand{\dtsa}{\textsc{\scalefont{1.2}dts}\,}
\newcommand{\dtse}{\textsc{\scalefont{1.2}dts$^\star$}\,}
\newcommand{\valunc}[2]{%
  \textnormal{#1}%
  {\scriptsize\,\,$\pm$\,\textnormal{#2}}%
}

\definecolor{highlight}{RGB}{225,230,255}
\newcommand{\highlight}[1]{\cellcolor{blue!10}{#1}}
\newtcolorbox{expblock}[2][]
{%
  enhanced,
  colback        = gray!2,          
  colframe       = gray!65!black,   
  coltitle       = white,           
  boxed title style = {             
    colback = gray!14,
    size    = small,
    boxrule = 0pt,                  
  },
  arc           = 1.6mm,            
  boxrule       = .4pt,             
  left          = 1.8mm,
  right         = 1.8mm,
  top           = 1.2mm,
  bottom        = 1.2mm,
  fonttitle     = \bfseries\normalsize,
  title         = {#2},
  #1
}

\makeatletter

\renewcommand{\appendixautorefname}{\S\@gobble}
\renewcommand{\sectionautorefname}{\S\@gobble}
\renewcommand{\subsectionautorefname}{\S\@gobble}
\renewcommand{\subsubsectionautorefname}{\S\@gobble}
\makeatother
\makeatletter
\providecommand{\section}{}
\renewcommand{\section}{%
  \@startsection{section}{1}{\z@}%
                {-1.5ex \@plus -0.5ex \@minus -0.2ex}%
                { 1.0ex \@plus  0.3ex \@minus  0.2ex}%
                {\large\bf\raggedright}%
}
\providecommand{\subsection}{}
\renewcommand{\subsection}{%
  \@startsection{subsection}{2}{\z@}%
                {-1.3ex \@plus -0.5ex \@minus -0.2ex}%
                { 0.3ex \@plus  0.2ex}%
                {\normalsize\bf\raggedright}%
}
\providecommand{\subsubsection}{}
\renewcommand{\subsubsection}{%
  \@startsection{subsubsection}{3}{\z@}%
                {-1.0ex \@plus -0.5ex \@minus -0.2ex}%
                { 0.3ex \@plus  0.2ex}%
                {\normalsize\bf\raggedright}%
}
\providecommand{\paragraph}{}
\renewcommand{\paragraph}{%
  \@startsection{paragraph}{4}{\z@}%
                {0.2ex \@plus 0.2ex \@minus 0.2ex}%
                {-1em}%
                {\normalsize\bf}%
}
\providecommand{\subparagraph}{}
\renewcommand{\subparagraph}{%
  \@startsection{subparagraph}{5}{\z@}%
                {1.5ex \@plus 0.5ex \@minus 0.2ex}%
                {-1em}%
                {\normalsize\bf}%
}

\makeatother
\title{Diffusion Tree Sampling: Scalable inference‑time alignment of diffusion models}
\author{%
    Vineet Jain,\; Kusha Sareen,\; Mohammad Pedramfar,\; Siamak Ravanbakhsh\\
    School of Computer Science, McGill University\\
    Mila - Quebec Artificial Intelligence Institute\\
    \texttt{\{jain.vineet, siamak.ravanbakhsh\}@mila.quebec}
}

\begin{document}

\maketitle

\begin{abstract}
Adapting a pretrained diffusion model to new objectives at inference time remains an open problem in generative modeling. Existing steering methods suffer from inaccurate value estimation, especially at high noise levels, which biases guidance. Moreover, information from past runs is not reused to improve sample quality, resulting in inefficient use of compute. Inspired by the success of Monte Carlo Tree Search, we address these limitations by casting inference-time alignment as a search problem that reuses past computations. We introduce a tree-based approach that \emph{samples} from the reward-aligned target density by propagating terminal rewards back through the diffusion chain and iteratively refining value estimates with each additional generation. Our proposed method, Diffusion Tree Sampling (\dtsa\!), produces asymptotically exact samples from the target distribution in the limit of infinite rollouts, and its greedy variant, Diffusion Tree Search (\dtse\!), performs a global search for high reward samples. On MNIST and CIFAR-10 class-conditional generation, \dtsa matches the FID of the best-performing baseline with up to $10\times$ less compute. In text-to-image generation and language completion tasks, \dtse effectively searches for high reward samples that match best-of-N with up to $5\times$ less compute. By reusing information from previous generations, we get an \emph{anytime algorithm} that turns additional compute into steadily better samples, providing a scalable approach for inference-time alignment of diffusion models\footnote{Project page: \href{https://diffusion-tree-sampling.github.io}{\texttt{https://diffusion-tree-sampling.github.io}}}.
\end{abstract}

\section{Introduction}
\label{sec:intro}

\begin{figure}[t!]
    \centering
    \begin{subfigure}[b]{.99\textwidth}
        \centering
        \includegraphics[width=\textwidth,trim={6em 2.5em 1.5em 9.25em},clip]{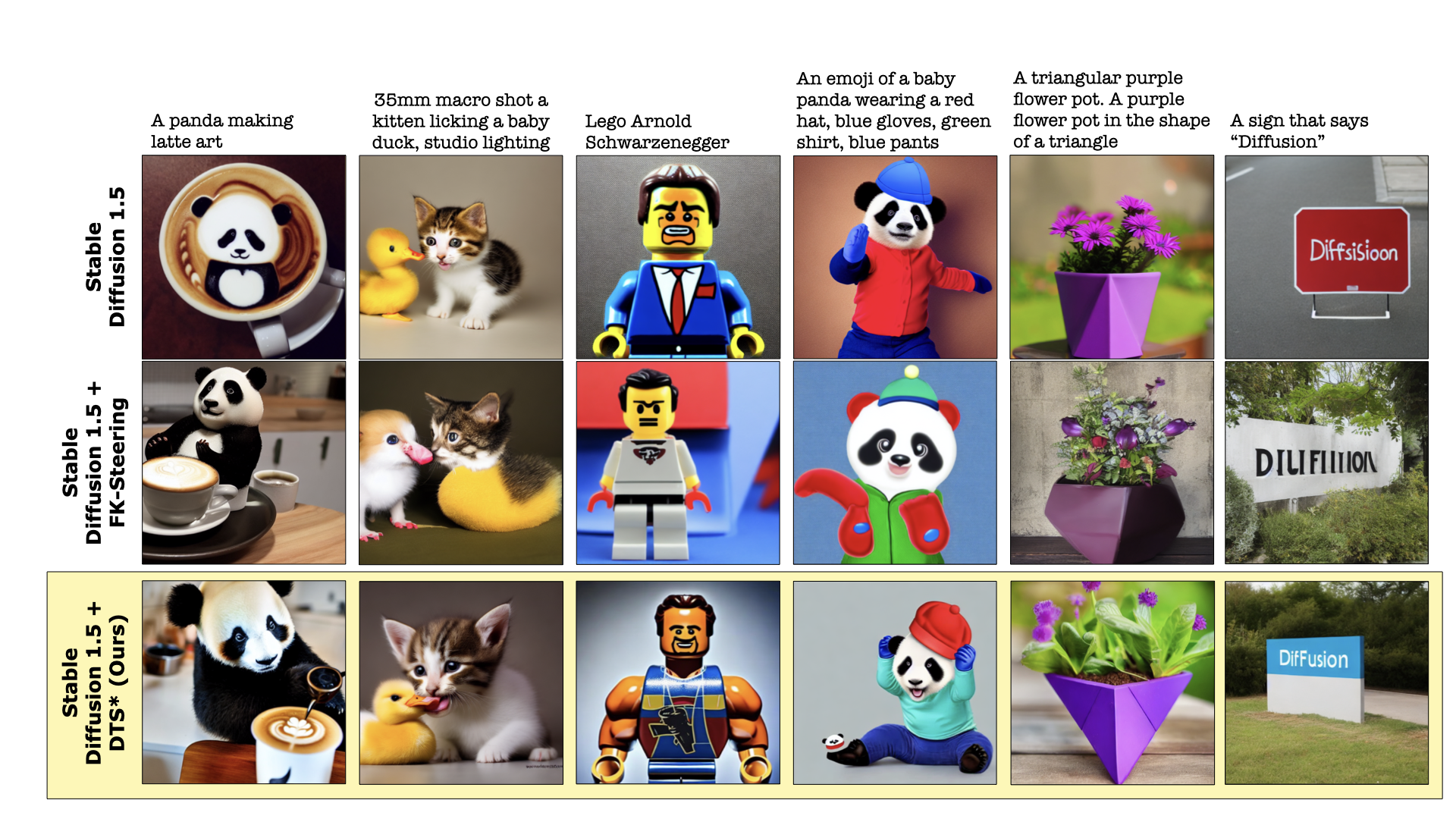}
        \label{fig:sd_images}
    \end{subfigure}
    \vspace{-0.5em}
    \caption{Sample text-image pairs using Stable Diffusion v1.5 \citep{rombach2022high} and ImageReward \citep{xu2023imagereward} as the guiding function, with generated samples picked at random for each method and prompt.}
    \vspace{-1em}
    \label{fig:main_fig}
\end{figure}

Diffusion models have emerged as one of the most powerful frameworks for generative modeling, achieving state-of-the-art results across a wide range of modalities, including image synthesis \citep{ho2020denoising, song2021scorebased, rombach2022high}, molecule conformer generation \citep{hoogeboom2022equivariant,xugeo2022diff}, and text generation \citep{sahoo2024simpleeffectivemaskeddiffusion,lou2024discrete}. 
Despite their success, adapting a pretrained diffusion model to satisfy new, user-defined objectives at inference time without expensive retraining or fine-tuning remains a major challenge \citep{uehara2025inference}.

Most objectives can be cast as a reward function, turning alignment into a posterior sampling problem where the target is to sample from the pretrained model density weighted by exponentiated reward. The key challenge is that rewards are only available at the end of the denoising trajectory. So, inference-time alignment seeks to \emph{guide the denoising process based on unseen terminal rewards}. 

A range of different methods have been proposed -- gradient-based guidance \citep{dhariwal2021diffusion, chung2023dps, bansal2023universal}, where one uses reward gradients to bias the denoising process; sequential Monte Carlo (SMC) \citep{wu2023practical,trippe2023diffusion,cardoso2024monte,dou2024diffusion,kim2025test} which maintains a population of particles and resamples them during denoising based on an estimate of terminal rewards; or more recently, search-based methods \citep{li2024derivative, li2025dynamic, ma2025inference} that perform a local greedy search based on approximate rewards. 
The common issue undermining all of these methods is that they rely on certain approximations to estimate the unseen terminal rewards. As we demonstrate in \Cref{sec:pitfalls}, these approximations bias decisions and degrade sample quality.

We therefore address the following challenges or questions in this work: (1) how to guide the diffusion process at inference-time when rewards are available only at the end? This is also known as the credit assignment problem in reinforcement learning (RL) literature \citep{minsky1961steps}; (2) inference-time samples can potentially inform and improve future samples -- how to systematically use this information in a sequential yet scalable sampling process?

\begin{wrapfigure}{r}{0.45\textwidth}
\vspace{-1.2em}
\centering
\includegraphics[width=\linewidth,trim={2em 1em 3em 2em},clip]{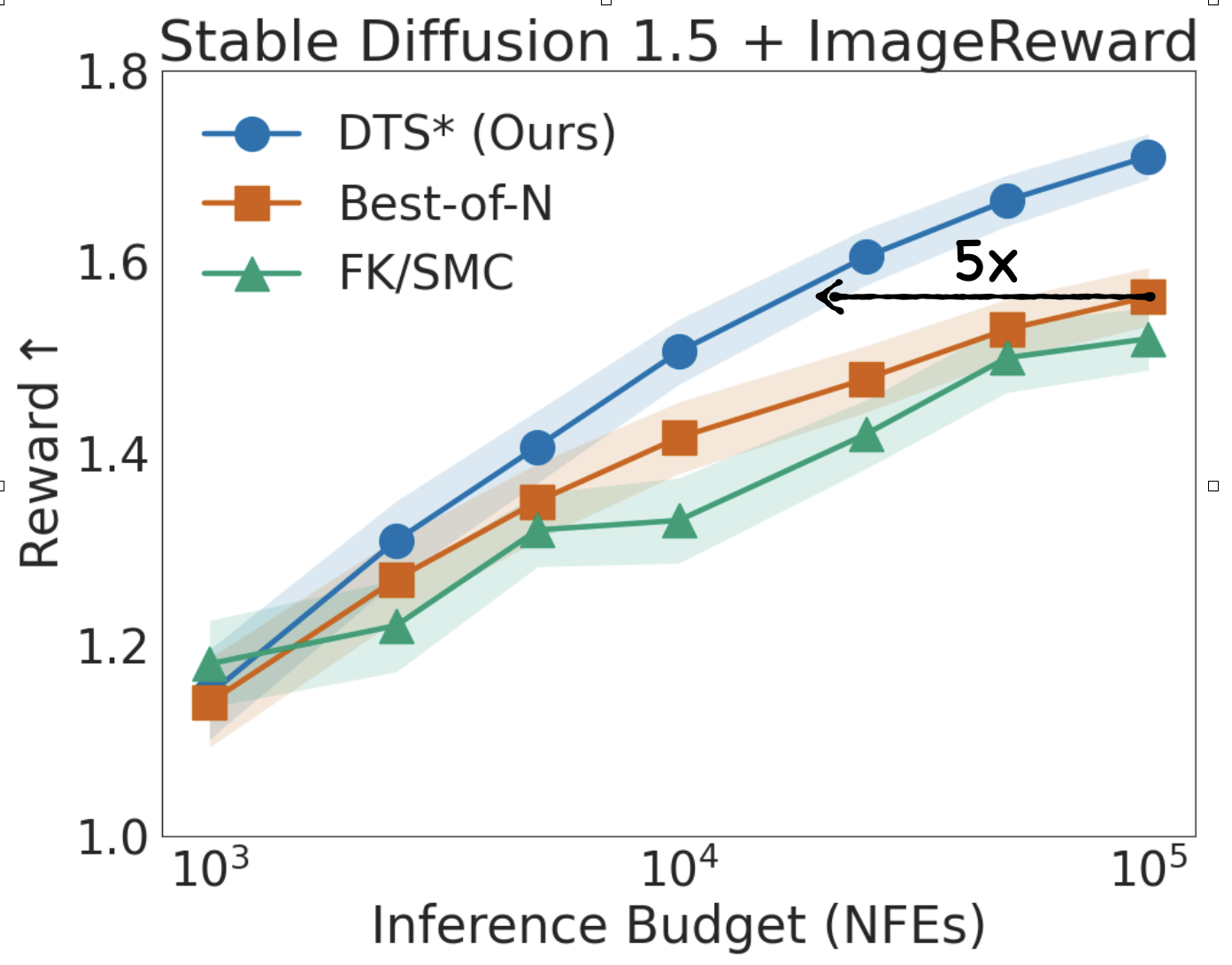}
\vspace{-1.2em}
\caption{Scaling inference compute.}
\vspace{-1.5em}
\label{fig:intro_plot}
\end{wrapfigure}

Fortunately, RL also provides a solution that has been historically quite successful in addressing both challenges -- Monte Carlo Tree Search (MCTS) \citep{browne2012survey}. We therefore ask: can we leverage MCTS for steering diffusion models?
We observe that during denoising, the pretrained diffusion model can be viewed as a deterministic policy, while the reverse process Gaussian step can be viewed as a stochastic environment transition. This is exactly the classical setting for MCTS \cite{kocsis2006bandit}, suggesting we could use tree search to solve the problem of inference-time alignment.


Our proposed algorithm, Diffusion Tree Sampling (\dtsa\!) is a novel inference-time alignment method that casts the denoising process as a finite-horizon tree, where similar to MCTS, rollouts are used to continuously improve value estimates for intermediate noisy states. For applications that require optimization, rather than sampling from the target density, we propose a search variant -- Diffusion Tree Search (\dtse\!) -- that performs a principled search in the space of denoising trajectories to identify the modes within high-volume regions of the target density.

Our contributions can be summarized as follows:
\begin{itemize}[leftmargin=1.8em,itemsep=0.3em]
\item We formulate inference-time alignment of diffusion models as a tree search problem for sampling from the reward-aligned distribution or optimizing for high reward samples. 
\item We develop a general tree-based algorithm that yields asymptotically exact samples from the target distribution in the limit of infinite rollouts.
\item We demonstrate that \dtsa significantly reduces bias and variance in value estimation compared to common approximations used by many existing methods.
\item We show that both \dtsa and \dtse scale more favorably compared to leading baselines and match their performance with up to $10\times$ less compute on class-conditional image generation, and up to $5\times$ less compute on text-to-image alignment and language completion tasks.
\end{itemize}

\FloatBarrier
\clearpage
The paper is organized as follows. \Cref{sec:prelim} reviews diffusion models and the formalism of inference-time alignment. \Cref{sec:pitfalls} discusses relevant literature and investigates the issues with existing approaches. Section~\ref{sec:dts} presents the \dtsa and \dtse algorithms, along with theoretical guarantees of consistency. Section~\ref{sec:exp} provides empirical results in various high-dimensional settings. We conclude in Section~\ref{sec:discussion} with a discussion of limitations, computational considerations, and future directions.

\section{Preliminaries}
\label{sec:prelim}

\paragraph{Diffusion models. }
\label{sec:prelim_1}
Diffusion models \citep{ho2020denoising, song2021scorebased} define a generative process via a Markov chain that progressively adds noise to data $\bx_0 \sim p_{\text{data}}(\bx)$, referred to as the forward process,
\begin{align*}
\bx_t = \sqrt{\alpha_t}\, \bx_{t-1} + \sqrt{1 - \alpha_t}\, \epsilon, \quad \epsilon \sim \mathcal{N}(0, I),
\end{align*}
where $t \in \{1,\dots,T\}$ indexes discrete time steps, and $\{\alpha_t\}_{t=1}^T$ defines a noise schedule. The noise schedule is chosen such that at $t=T$, the marginal distribution of the samples resembles a simple fixed distribution, such as standard Gaussian, $p(\bx_T) = \mathcal{N}(0, I)$. A learned reverse process iteratively denoises samples:
\begin{align*}
p_\theta(\bx_{t-1} \mid \bx_t) = \mathcal{N}(\bx_{t-1}; \mu_\theta(\bx_t,t), \sigma_t^2 I),
\end{align*}
where $\sigma_t$ is the posterior variance calculated from the forward noise schedule and $\mu_\theta$ is parameterized typically by neural networks and optimized by minimizing the variational bound on the data likelihood or via denoising score matching. The generative process induces a distribution:
\begin{align}
    p_\theta(\bx_0,\ldots,\bx_{T-1},\bx_T) = p(\bx_T) \prod_{t=1}^T p_\theta(\bx_{t-1} \mid \bx_t), \quad p(\bx_T) = \mathcal{N}(0, I).
\end{align}

\paragraph{Alignment of diffusion models.}
\label{sec:prelim_2}

Consider a pretrained diffusion model $p_\theta$ and an optimality variable $\mathcal{O} \in \{0,1\}$ which denotes whether a sample $\bx \sim p_\theta(\bx)$ satisfies some desirable property. This setting naturally emerges in domains such as molecular design, where we might want samples to have specific chemical properties, or text-to-image generation, where samples should match text prompts \citep{rombach2022high}. 
This is equivalent to sampling from the posterior distribution 
\begin{align*}
    \mbox{$p(\bx \mid \mathcal{O} = 1) \propto p_\theta(\bx) \,p(\mathcal{O} = 1 \mid \bx)$}.
\end{align*}
A typical assumption is that $p(\mathcal{O} = 1 \mid \bx) \propto \exp(\lambda r(\bx))$ where $r$ is some reward function and $\lambda$ is the inverse temperature. For the rest of this paper, we assume that $\lambda=1$ unless otherwise stated and we define the alignment problem as sampling from the target distribution or finding its mode, where ${Z}$ is the normalization constant:
\begin{align}
    \pi^*(\bx) = \frac{1}{{Z}}\; p_\theta(\bx) \exp(\lambda r(\bx)). \label{eq:product}
\end{align}

\paragraph{Reinforcement learning approach.}
\label{sec:prelim_3}
Since the generative process in diffusion models defines a Markov chain, we may consider the model $p_\theta$ as a policy. The target distribution $\pi^*$ can be seen as the optimal policy for the following objective:
\begin{align}
    \pi^*(\bx) \coloneqq \argmax_\pi \mathbb{E}_{\bx \sim \pi(\cdot)}\left[r(\bx)\right] - \frac{1}{\lambda} \,D_\textrm{KL}\left(\pi \,\|\, p_\theta\right).
\end{align}
This is closely related to the maximum entropy RL objective \citep{ziebart2008maximum, haarnoja2017reinforcement},
except that the entropy regularization (which equals the KL divergence with respect to a uniform policy) is replaced by the KL divergence with the pretrained model $p_\theta$.
We define the \textit{soft value function} at timestep $t$ as the expected exponentiated reward starting from $\bx_t$ and following $p_\theta$:
\begin{align}
\label{eq:soft_value}
V_t(\bx_t) &\coloneqq \frac{1}{\lambda}\log \mathbb{E}_{p_\theta(\bx_{0:t-1}\mid\bx_t)}\left[\exp\left(\lambda r(\bx_0)\right)\right].
\end{align}
This soft value function satisfies the following recursive relation, analogous to the soft Bellman equation and exactly characterizes the optimal policy $\pi^*$:
\begin{align} 
V_t(\bx_t) &= \frac{1}{\lambda}\log \mathbb{E}_{p_\theta(\bx_{t-1}\mid\bx_t)}\left[\exp\left(\lambda V_{t-1}(\bx_{t-1})\right)\right], \quad V_0(\bx_0) = r(\bx_0).
\label{eq:soft_bellman}
\end{align}
\vspace{-1em}
\begin{align} 
\pi_t^*(\bx_{t-1}\mid\bx_t) &= \frac{p_\theta(\bx_{t-1}\mid\bx_t)\exp\left(\lambda V_{t-1}(\bx_{t-1})\right)}{\int p_\theta(\bx_{t-1}\mid\bx_t)\exp\left(\lambda V_{t-1}(\bx_{t-1})\right) d\bx_{t-1}}.
\label{eq:optimal_policy}
\end{align}
This formulation explicitly connects optimal sampling with soft value estimation, motivating various practical approximations and sampling methods discussed in subsequent sections. For completeness, we derive \Cref{eq:soft_bellman,eq:optimal_policy} in \Cref{app:proof_1}. In the rest of the paper, we use $V_t$ to denote the true soft value function and $\hat{v}_t$ to denote estimates.

\section{Inference-time adaptation of diffusion models}
\label{sec:pitfalls}

One option to obtain the optimal policy in \Cref{eq:optimal_policy} is to train the diffusion model using RL, which is called fine-tuning \citep{fan2023dpok,black2024training, venkatraman2024amortizing,domingo2024adjoint}. This idea is not compatible with a priori unknown reward functions presented at inference-time. Sampling from unseen reward functions would require guiding the denoising process during inference to align with the optimal policy without modifying the prior pretrained model. There is a growing body of work in this direction, we discuss some of the most relevant works below, and works in related areas in \Cref{app:related}.


\paragraph{Gradient-based guidance.} 
One way to sample from the optimal policy $\pi^*$ is to use the first-order Taylor expansion of $V_{t-1}$ around the pretrained mean $\mu_{\theta}(\bx_{t},t)$.
This yields the gradient-based denoising step $\tilde{\bx}_{t-1} \!\sim\! \mathcal{N}\left(\mu_{\theta}(\bx_{t},t)+\lambda\,\sigma_{t}^{2}\,\nabla_{\bx_{t-1}} V_{t-1}(\bx_{t-1}), \sigma_t^2\,I\right)$.
This can be considered a form of classifier guidance \citep{dhariwal2021diffusion} and is used in many proposed inference-time steering methods \citep{chung2023dps, bansal2023universal, he2024manifold}. The gradient approximation can be improved by using Monte Carlo samples for estimation \citep{song2023loss}.

\paragraph{Sequential Monte Carlo.}
Particle-based methods are another very popular approach, where a population of samples is maintained to approximately sample from the desired distribution. Sequential Monte Carlo (SMC) \citep{del2006sequential} uses potential functions, which usually approximate the soft value function, to assign weights to particles and resample them at every step. Different variations of SMC have been proposed for diffusion model alignment \citep{wu2023practical,trippe2023diffusion,cardoso2024monte,dou2024diffusion,kim2025test}. Classical SMC guarantees exact sampling in the limit of infinite particles and exact value estimation. In practice, however, the repeated sampling procedure can reduce diversity due to weight variance and inaccurate value estimates. We provide detailed background on SMC for diffusion sampling in \Cref{app:smc_background}.

\paragraph{Search-based methods.}
Recently, there has been a growing interest in using search-based methods to align diffusion models \citep{ma2025inference}. Most of these methods propose doing a local search \citep{li2024derivative, li2025dynamic} by obtaining multiple denoising candidates at each step and selecting the best one based on their value. More recently, tree search has been combined with best-of-N \citep{zhang2025t}, and an MCTS-based approach \citep{yoon2025monte} has been applied in the specific context of diffusion forcing \citep{chen2024diffusion} over sequences for planning. However, these methods either do not use an explicit backup mechanism, resulting in a limited local search\citep{li2024derivative, li2025dynamic, zhang2025t, ma2025inference}, or they rely on inaccurate value estimates \citep{yoon2025monte}. \dtsa, on the other hand, performs global credit assignment using all trajectories for asymptotically exact sampling.

\subsection{The value estimation problem}
\label{sec:value}

\begin{figure}[b]
    \centering
    \includegraphics[width=.99\linewidth]{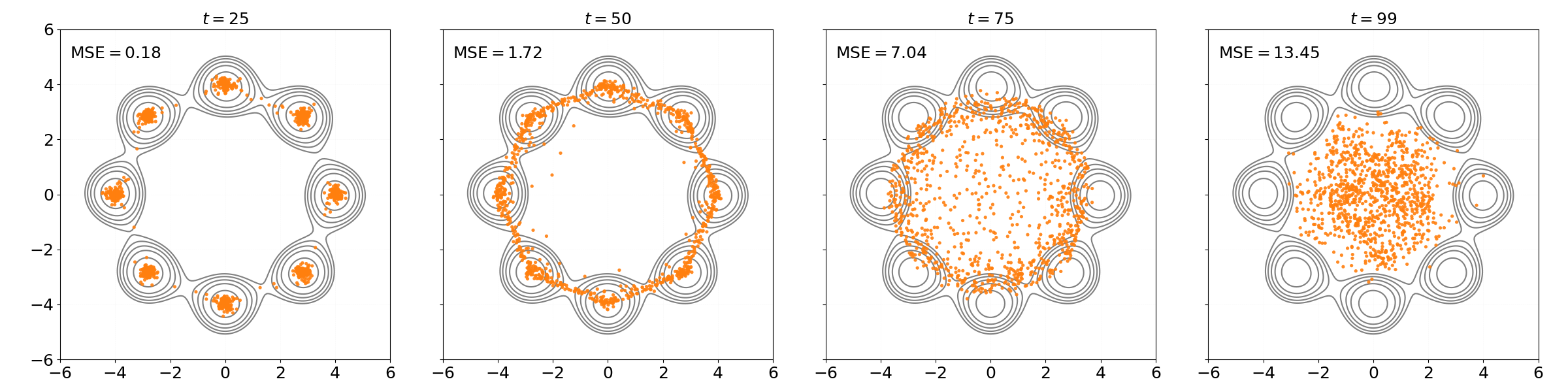}
    \caption{One-step prediction $\hat{\bx}_0(\bx_t)$ using Tweedie's formula for different time steps, along with average mean squared error with the ground truth data samples. Close to $t=0$, the predictions are fairly accurate, but towards the maximum timestep $T=99$, they devolve into random predictions.}
    \vspace{-1em}
    \label{fig:onestep_denoising}
\end{figure}

Accurate credit assignment requires estimating the value function from \Cref{eq:soft_bellman} at each intermediate state. If the reward function is known a priori, we could use temporal difference learning to train value networks. Most existing inference-time alignment methods employing gradient guidance \citep{chung2023dps,song2023loss}, SMC \citep{wu2023practical,kim2025test}, or search \citep{li2024derivative,li2025dynamic,yoon2025monte} rely on the following set of approximations to circumvent this problem. The first step is to apply Jensen's inequality:
\begin{align}
    V_t(\bx_t) = \frac{1}{\lambda}\log \mathbb{E}_{p_\theta(\bx_{0:t-1}\mid\bx_t)}\left[\exp\left(\lambda \,r(\bx_0)\right)\right] \approx \mathbb{E}_{p_\theta(\bx_{0:t-1} \mid \bx_t)}\left[r(\bx_0)\right]
\end{align}
Estimation of this expectation is computationally expensive since it requires multiple rollouts from the current sample at timestep $t$ to the clean sample. The expected reward is further approximated as $\mathbb{E}_{p_\theta(\bx_{0:t-1} \mid \bx_t)}\left[r(\bx_0)\right] \approx r(\hat{\bx}_0(\bx_t))$, where $\hat{\bx}_0$ is the posterior mean obtained using Tweedie's formula \citep{efron2011tweedie,chung2023dps} in a single step:
\begin{align}
    \hat{\bx}_0(\bx_t) = \mathbb{E}_{p_\theta(\bx_{0:t-1} \mid \bx_t)}\left[\bx_0\right] = \frac{1}{\sqrt{\alpha_t}}\,\left(\bx_t + (1 - \bar{\alpha}_t) \nabla_{\bx_t} \log p_t(\bx_t)\right).
\end{align}
The posterior mean is an approximation because the true score function for intermediate marginal densities $\nabla_{\bx_t} \log p_t(\bx_t)$ is replaced by the learned score function. We investigate the effect of this approximation for a diffusion model trained on a mixture of Gaussians in \Cref{fig:onestep_denoising}. The prediction based on Tweedie's formula is quite accurate for low time steps, but gets increasingly inaccurate as we go towards noisy states, eventually degrading into random predictions. Therefore, despite the wide adoption of this approximation, the value estimates used for guidance are essentially random at higher-noise levels even in simple 2D settings.  

\subsection{Scaling axes}
Efficient utilization of available compute is critical for any inference-time alignment algorithm. Existing SMC or search-based methods treat each sampling procedure as an independent event, and all intermediate evaluations are discarded. 
Consider a streaming or repeated sampling setting. There is no mechanism to assimilate information from prior runs to improve sample quality. This could be particularly useful for correcting errors in value estimation, which, as we saw above, are difficult to estimate at high noise levels. As a result, these methods scale along a single axis -- particle count -- and cannot turn extra compute into cumulative improvements in estimate quality. 




\section{Diffusion Tree Sampling and Search}
\label{sec:dts}

The pitfalls above suggest two complementary desiderata for an effective inference–time sampler:
\begin{enumerate}[label=(D\arabic*),leftmargin=2.8em,itemsep=0.3em]
\item Use information from low-noise timesteps, where the reward signal is reliable, to \emph{refine decisions made at high-noise timesteps}, rather than treating every step in isolation.
\item Reuse information from previously explored trajectories so that \emph{additional compute improves sample quality} instead of merely increasing parallel particle count (this property is characteristic of an anytime algorithm).
\end{enumerate}
To address these issues, in this section we develop a solution by first interpreting the denoising process as a tree in \Cref{sec:dts_tree}. 
We then introduce a general tree-based algorithm to sample from the target density in \Cref{sec:dts_sampling} 
and describe its application to diffusion model alignment in \Cref{sec:dts_design}.
Finally, in \Cref{sec:dts_exp} we empirically evaluate our method on 2D datasets to validate, in a very clear and controlled setting, whether \dtsa satisfies the desiderata mentioned above.

\subsection{Denoising tree}
\label{sec:dts_tree}
The Markov property of the reverse diffusion chain naturally induces a finite horizon tree in $\mathbb{R}^d$, where $d$ is the dimensionality of the space over which we are diffusing. Here, the nodes at depth $t$ represent noisy states $\bx_t$ and the edges represent a denoising step. Each node $\bx_t$ can be stochastically denoised into multiple children $\bx_{t-1} \sim p_\theta(\cdot \mid \bx_t)$.

This framing allows us to keep track of information across multiple denoising trajectories, including estimates of the soft value function, which helps with global credit assignment. 

Using the tree structure gives us the flexibility to sample from the target density or search for the highest reward sample with minimal changes to the underlying algorithm. We call the sampling variant \textsl{Diffusion Tree Sampling} (\dtsa\!) and the search variant \textsl{Diffusion Tree Search} (\dtse\!).

\subsection{Tree-based sampling}
\label{sec:dts_sampling}

Similar to MCTS, we construct a tree ${\mathcal{T}}$,  where nodes represent states $\bx_t$ and edges represent transitions $p_\theta(\bx_{t-1} \mid \bx_t)$ following the base diffusion model.
Each node maintains the current state and timestep $(\bx_t, t)$, an estimate of the soft value function $\hat{v}(\bx_t)$, and the visit count $N(\bx_t)$. Since we do not have a fixed starting state, we introduce a dummy state as the root $\bx_{T+1}$ that transitions to the prior in our diffusion model -- i.e., $p(\bx_T \mid \bx_{T+1}) = \mathcal{N}(\mathbf{0}, I)$. 
Additionally, we use $\mathcal{C}(\bx_t)$ to denote the set of children of node $\bx_t$. 

\begin{figure}
    \centering
    \includegraphics[width=0.9\linewidth]{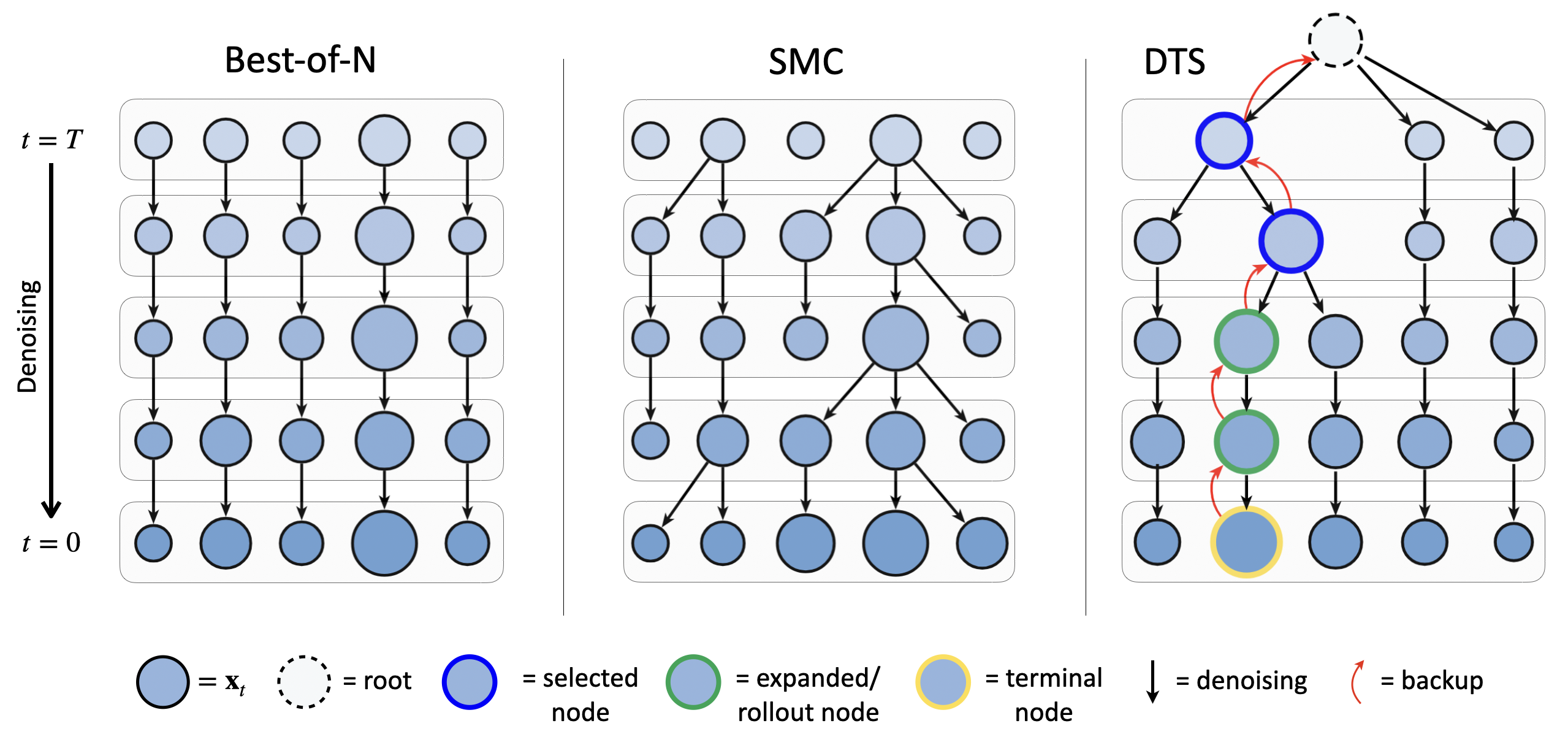}
    \caption{Illustration of various inference-time steering methods, where size of the node represents the associated values. \textbf{Left:} Best-of-N denoises multiple samples using the base diffusion model and selects the one with the highest reward. \textbf{Center:} SMC maintains a population of particles and resamples based on an estimate of the value function. \textbf{Right:} \dtsa and \dtse maintain a tree that accumulates information across multiple rollouts and backs up the terminal reward to refine value estimates. The diagram illustrates the four phases: selection, expansion, rollout, and backup.}
    \vspace{-1em}
    \label{fig:schematic}
\end{figure}

The goal is to expand this tree while improving value estimates as we expand it, so that it can be used for approximate sampling from the target distribution at any time during the construction. The resulting tree sampling process provably samples from the target distribution $\pi^*$ in the limit of infinite rollouts.
The tree-building procedure of \dtsa repeats the following steps iteratively:
\begin{enumerate}[leftmargin=1.8em,itemsep=0.3em]
    \item \textbf{Selection.}\; Starting from the root $\bx_0$, sample a child $\bx_{t-1} \in \mathcal{C}(\bx_t)$ from Boltzmann distribution of the values $\propto \exp\left(\lambda\,\hat{v}(\bx_{t-1})\right)$
    recursively until either an unexpanded node is reached or $t=0$.
    
    \item \textbf{Expansion.}\; If we reach a node $\bx_t$ such that the number of children is less than the maximum allowed value and $t>0$, we create a new child node $\bx_{t-1} \sim p_\theta(\cdot\mid\bx_t)$ and initialize $\hat{v}(\bx_{t-1}) = 0,\; N(\bx_{t-1}) = 1$.
    
    \item \textbf{Rollout.}\; From the newly created node, we perform a rollout till terminal states $\bx_0$ by recursively sampling from $p_\theta(\cdot \mid \bx_{t'})$ for $t' = t-1,\ldots,0$. An important distinction from traditional MCTS is that we add the rollout path to $\mathcal{T}$. 
    
    \item \textbf{Backup.}\; Evaluate the terminal node using the reward function $\hat{v}(\bx_0) = r(\bx_0)$ and use soft Bellman equation (\Cref{eq:soft_bellman}) to update parent node values using the children node values
    recursively for $t=0,\ldots,T$. The visit counts for all nodes in the path are also updated.
\end{enumerate}
Each traversal of the tree, from the root to the backup of the value function, constitutes one tree-building iteration. For sampling from $\mathcal{T}$, we simply start from the root and perform selection steps until we reach a terminal node. A formal algorithm is provided in \Cref{app:algo}.

\vspace{0.25em}
\begin{proposition}[Asymptotic consistency]
\label{prop:dts_proof}
Let $r$ be bounded and $\lambda > 0$, then \dtsa produces a sequence of terminal states whose empirical distribution converges to the optimal policy $\pi^\ast$ as the number of tree iterations $M\to\infty$.
\end{proposition}\vspace{-0.5em}
\begin{proof}[Proof sketch]
By construction, the tree policy selects $\bx_{t+1}$ with unnormalized probability $p_\theta(\bx_{t+1}\mid\bx_t) \,\exp(\lambda \hat{v}(\bx_{t+1}))$, which is the optimal policy defined in \Cref{eq:optimal_policy}. By telescoping the product over $t$, we obtain the final samples at $t=T$ are sampled from $p_\theta(\bx) \exp(r(\bx_0))$. A more detailed proof is given in \Cref{app:proof}.
\end{proof}

\subsection{Design choices for diffusion alignment}
\label{sec:dts_design}

The algorithm discussed above can be applied to any Markov chain. However, in this work, we apply it to the problem of inference-time alignment of diffusion models. We discuss various considerations and design choices below, with more implementation details in \Cref{app:implementation}.

\paragraph{Sampling or Search.} 
\dtsa is designed to sample from the target distribution $\pi^*$, but for settings where a single high‑quality sample is required, we introduce a \emph{search} variant, \dtse.
It keeps the same soft value backup but modifies the selection step by always selecting the child with the largest soft value estimate instead of Boltzmann sampling. Since \dtse uses soft values, this is different from standard MCTS -- it implements a \emph{marginal‑MAP} (max–sum) inference scheme \citep{robert1999marginal} over the tree. At every noise step, the algorithm selects the branch whose subtree carries the greatest mass under $\pi^\star$ and, once $t=0$ is reached, returns the highest‑density leaf inside that dominant region. As we will see in the \Cref{sec:exp}, this volume-based selection helps avoid reward over-optimization. 

\paragraph{Branching.} 
Extensions of MCTS to continuous spaces commonly use \emph{progressive widening} \citep{couetoux2011continuous} to decide the maximum number of branches $B(\bx_t)$ allowed per node based on the number of visits:
    $B(\bx_t) = C \,\cdot\, N(\bx_t)^\alpha, \; C > 0,\; \alpha \in (0,1).$
The high-level intuition is that nodes that are visited more often should be expanded more, since they represent more promising directions for denoising. We adopt the same strategy and during tree traversal, if we encounter a node such that $|\mathcal{C}(\bx_t)| < B(\bx_t)$ and $t>0$, we will always expand. 

\paragraph{Exploration.} There is a rich literature on search methods for classical MCTS, and the most popular approach, UCT \citep{kocsis2006bandit}, is an application of upper-confidence bounds \citep{auer2002finite} to trees. We employ this exploration strategy for \dtse, i.e. we choose the child $\bx_{t-1} \in \mathcal{C}(\bx_t)$ with the maximum value of the UCT estimate:
\begin{align}
    \operatorname{UCT}(\bx_{t-1}) = \hat{v}(\bx_{t-1}) + c_\text{uct}\,\sqrt{\frac{\log{N(\bx_t)}}{N(\bx_{t-1})}}, \qquad c_\text{uct} > 0.
\end{align}
For \dtsa\!\!, we do not employ explicit exploration, because, in practice we observe that sampling obviates the need for an exploration bonus or handcrafted mechanism.

\paragraph{Efficient implementation.} The main computational cost is incurred when using the diffusion model proposal to sample new children or perform rollouts. We implement an efficient batched version of the algorithm by collecting nodes by timestep and performing one forward pass for all nodes at the same timestep. The selection and backup steps involve simple tensor operations and pointer manipulation with negligible cost. Therefore, while the control flow of our method is sequential, the practical algorithm can be parallelized. Note that once the tree has been built, sampling is near instantaneous by repeatedly selecting children without any model calls. 

\subsection{Illustrative experiments}
\label{sec:dts_exp}

\begin{figure}[b]
    \centering
    \includegraphics[width=\linewidth]{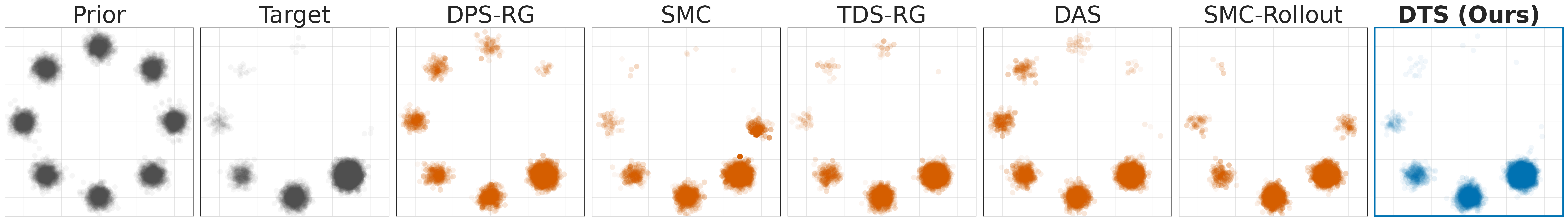}
    \includegraphics[width=\linewidth,trim={0em 0em 0em 3.5em},clip]{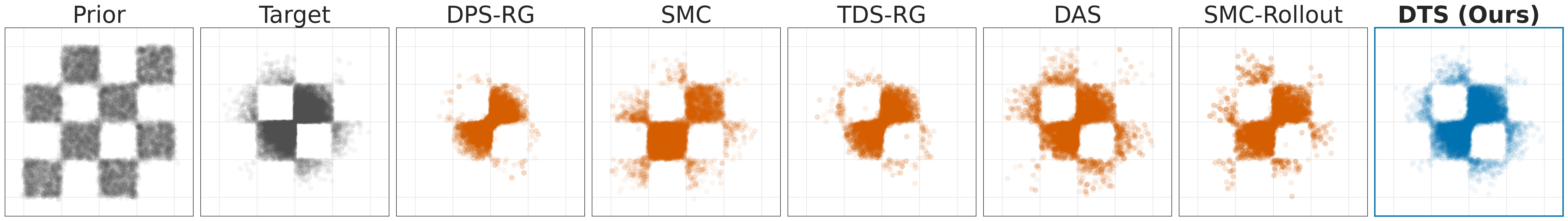}
    \caption{ Samples from the prior $p(\bx_0)$, target $p(\bx_0)\exp(r(\bx_0))\,/\,\mathcal{Z}$ and different sampling methods at $10^6$ NFEs. \textbf{Top:} The prior is an equal-weighted mixture of Gaussians, and the reward function distributes mass unevenly. \textbf{Bottom:} The prior has support on alternate square regions in a checkerboard pattern, and the reward function $r(x,y) = -0.5(x^2 + y^2)$ is negative distance from the origin.}
    \label{fig:2d_dist}
\end{figure}

\begin{figure}[t]
    \centering
    \includegraphics[width=\textwidth]{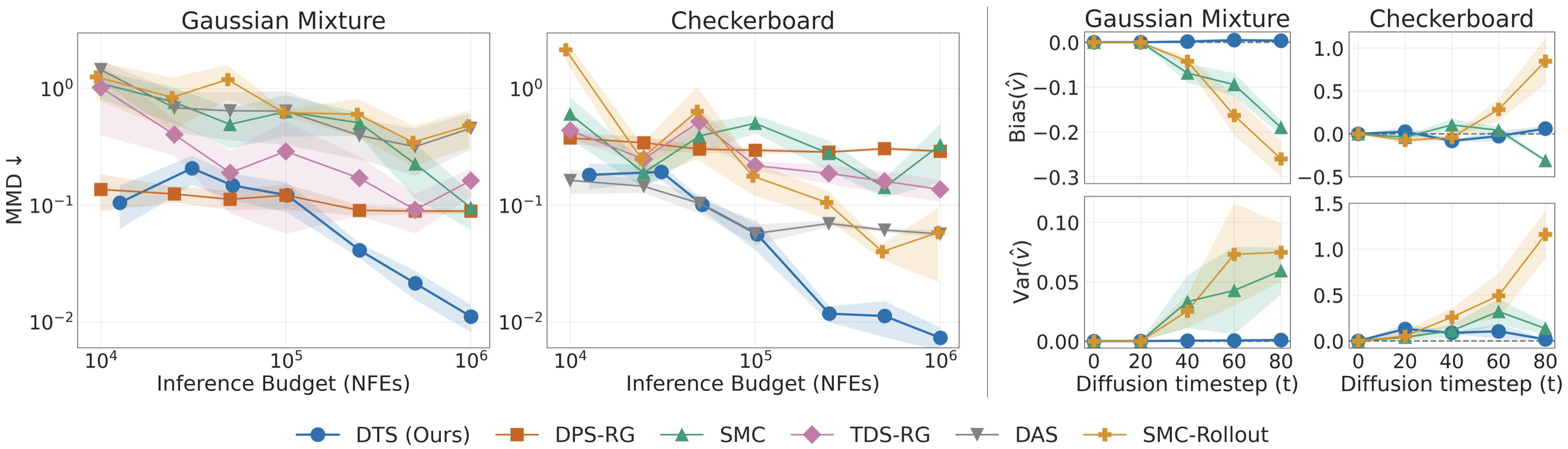}
    \caption{\textbf{Left:} Maximum mean discrepancy (MMD) between generated samples and target ground truth samples as a function of number of function evaluations of the prior diffusion model. \textbf{Right:} Bias and variance of value estimates for different approaches at $10^6$ NFEs.}
    \vspace{-1em}
    \label{fig:2d_metrics}
\end{figure}

In this section, we perform experiments on simple 2D settings to answer the following questions:
\begin{itemize}[leftmargin=1.8em,itemsep=0em]
    \item Does \dtsa sample accurately from the target distribution?
    \item Does reward backup in \dtsa result in more accurate value estimates (desideratum D1)?
    \item Does sample quality of \dtsa improve with more inference-time compute (desideratum D2)?
\end{itemize}

\paragraph{Setting.} Consider a pretrained diffusion model as the prior $p(\bx_0)$ and an associated reward function $r(\bx_0)$ defined on the support of $p$. The goal is to draw samples $\propto p(\bx_0)\exp(r(\bx_0))$. We consider two settings: (a) the prior is a mixture of eight Gaussians with uniform weights, while the reward function distributes mass unevenly to the different modes, and (b) the prior has alternate regions of support in a checkerboard pattern, and the reward function is the negative distance from the origin, $r(x,y) = -0.5(x^2 + y^2)$. More details on the experimental setup are provided in \Cref{app:implementation_exp}. 


\paragraph{Baselines.}

We compare \dtsa with several inference‑time steering methods including some which were originally proposed for posterior sampling in inverse problems, and adapt them to the reward‑guidance setting: (1) \textsl{DPS‑RG} (our reward‑guided version of DPS \citep{chung2023dps}) = gradient‑based guidance only, (2) \textsl{SMC} \citep{singhal2025general}, (3) \textsl{TDS‑RG} (reward-guided version of TDS \citep{wu2023practical}) = SMC + gradient guidance, (4) \textsl{DAS} \citep{kim2025test} = SMC + gradient guidance + tempering. We also implement a version of SMC, which we call \textsl{SMC-Rollout}, where the values are estimated via one full DDIM \citep{song2020denoising} rollout. For fair comparison, we benchmark all methods with respect to number of function evaluations (NFEs) of the diffusion model.


\paragraph{Results.} 


\Cref{fig:2d_dist} plots the samples obtained using different methods for two different settings. In both cases, \dtsa approximates the ground truth target density more accurately, with other methods distributing mass inaccurately to different areas of the support. 
In particular, gradient-based methods like DPS-RG and TDS-RG suffer from instability and require gradient clipping to stabilize denoising steps.
\Cref{fig:2d_metrics} shows the Maximum mean discrepancy (MMD) between generated samples and ground truth samples for an increasing amount of inference budget, measured in terms of the number of function evaluations of the diffusion model. 
This empirically validates that the sample quality of \dtsa improves with more compute, satisfying desideratum \textsc{D2}.

\paragraph{Bias-variance analysis of value estimates.} We estimate ground truth soft value estimates at different timesteps by performing $1000$ rollouts from noisy states using the base model and then taking log-sum-exp of the rewards. We then compute the relative mean squared error with the value estimates obtained using different approximations and decompose it into bias and variance. \Cref{fig:2d_metrics} shows this for different diffusion timesteps using \dtsa, Tweedie's formula (SMC + variants), and a single full DDIM rollout (SMC-Rollout). 
Both one-step denoising and single rollout have high bias and variance, which generally get worse for higher timesteps. Our tree-based approach reduces bias by using accurate reward information and reduces variance by aggregating information from multiple rollouts.
This empirically validates that \dtsa satisfies desideratum \textsc{D1}.


\section{Experiments}
\label{sec:exp}

We validate the efficacy of \dtsa and \dtse for image and text generation. Our experiments show that \dtsa draws faithful samples from high-dimensional posteriors, and \dtse efficiently searches high-dimensional image space to discover high-reward samples. We provide detailed description of each experimental setting in \Cref{app:implementation_exp}, details of baseline implementations in \Cref{app:baselines}, and additional results in \Cref{app:exp}.

\subsection{Class-conditional posterior sampling}
\label{sec:img_exp}

We evaluate \dtsa on the task of sampling from a class-conditioned posterior distribution $p(\bx \mid c) \propto p_\theta(\bx)p(c\mid \bx)$ where $p_\theta(\bx)$ is a pretrained unconditional diffusion model and $p(c\mid \bx)$ is a classifier. This would correspond to setting $r(\bx) = \log p(c\mid \bx)$ in \Cref{eq:product}.

\paragraph{Setting.} We use MNIST and CIFAR-10 datasets, each with 10 classes. In both cases, the priors are unconditional diffusion models in pixel-space -- we train one from scratch on MNIST and use an off-the-shelf model\footnote{\href{https://huggingface.co/google/ddpm-cifar10-32}{https://huggingface.co/google/ddpm-cifar10-32}}
for CIFAR-10. For MNIST, we consider two settings: sampling from individual digits, and sampling from even/odd digits. The latter is a multimodal posterior with reward function $r(\bx) = \max_{\{i=0,2,4,6,8\}} \log p(c=i\mid \bx)$ for the even digits and similarly for odd digits. For CIFAR-10, we sample from individual classes.

\begingroup
\setlength{\tabcolsep}{2pt}       
\renewcommand{\arraystretch}{1.2} 
\begin{table}[t]
\caption{Comparison of inference-time posterior sampling methods. We report the mean{\scriptsize $\pm$std} of each metric across the relevant classes and highlight $\pm5\%$ values from the best experimental value.}
\label{tab:img_results}
\centering
\resizebox{\columnwidth}{!}{%
\begin{tabular}{%
  l cccc cccc cccc
}
\toprule
Dataset →
  & \multicolumn{4}{c}{\textbf{MNIST}}
  & \multicolumn{4}{c}{\textbf{MNIST even/odd}}
  & \multicolumn{4}{c}{\textbf{CIFAR-10}} \\
\cmidrule(lr){2-5}\cmidrule(lr){6-9}\cmidrule(l){10-13}
Algorithm ↓ 
  & FID\,(↓)
  & CMMD\,(↓)
  & $\mathbb{E}[\log r(\bx)]$\,(↑)
  & Diversity\,(↑)
  & FID\,(↓)
  & CMMD\,(↓)
  & $\mathbb{E}[\log r(\bx)]$\,(↑)
  & Diversity\,(↑)
  & FID\,(↓)
  & CMMD\,(↓)
  & $\mathbb{E}[\log r(\bx)]$\,(↑)
  & Diversity\,(↑) \\
\midrule
DPS & \valunc{0.359}{0.227} & \valunc{0.441}{0.447} & \valunc{-0.323}{0.286} & \highlight{\valunc{0.474}{0.051}} & \valunc{0.123}{0.031} & \valunc{0.293}{0.139} & \valunc{-0.002}{0.001} & \valunc{0.572}{0.053} & \valunc{0.486}{0.121} & \valunc{2.609}{0.824} & \highlight{\valunc{-0.002}{0.001}} & \highlight{\valunc{0.551}{0.024}} \\
SMC/FK & \valunc{0.060}{0.051} & \valunc{0.177}{0.142} & \valunc{-0.002}{0.004} & \valunc{0.422}{0.040} & \valunc{0.027}{0.009} & \valunc{0.123}{0.113} & \valunc{-0.003}{0.003} & \highlight{\valunc{0.583}{0.084}} & \valunc{0.313}{0.070} & \valunc{1.409}{0.445} & \valunc{-0.102}{0.093} & \valunc{0.487}{0.045} \\
TDS & \valunc{0.087}{0.035} & \valunc{0.463}{0.260} & \highlight{\valunc{-0.001}{0.001}} & \valunc{0.404}{0.042} & \valunc{0.053}{0.010} & \valunc{0.250}{0.056} & \highlight{\valunc{-0.001}{0.000}} & \highlight{\valunc{0.576}{0.124}} & \valunc{0.487}{0.112} & \valunc{2.675}{0.665} & \valunc{-0.046}{0.055} & \valunc{0.469}{0.042} \\
DAS & \valunc{0.039}{0.017} & \valunc{0.179}{0.099} & \valunc{-0.016}{0.016} & \valunc{0.440}{0.041} & \valunc{0.031}{0.002} & \valunc{0.079}{0.011} & \valunc{-0.015}{0.019} & \highlight{\valunc{0.603}{0.094}} & \valunc{0.241}{0.037} & \valunc{0.822}{0.203} & \valunc{-0.584}{0.200} & \highlight{\valunc{0.530}{0.023}} \\
\midrule
\textbf{DTS (ours)} & \highlight{\valunc{0.014}{0.005}} & \highlight{\valunc{0.068}{0.030}} & \valunc{-0.023}{0.006} & \highlight{\valunc{0.452}{0.050}} & \highlight{\valunc{0.007}{0.003}} & \highlight{\valunc{0.036}{0.029}} & \valunc{-0.010}{0.004} & \highlight{\valunc{0.597}{0.069}} & \highlight{\valunc{0.195}{0.041}} & \highlight{\valunc{0.745}{0.201}} & \valunc{-0.305}{0.116} & \highlight{\valunc{0.542}{0.020}} \\
\bottomrule
\end{tabular}
}
\end{table}
\endgroup
\begin{figure*}[t]
    \centering
    \vspace{-0.5em}
    \includegraphics[width=\textwidth]{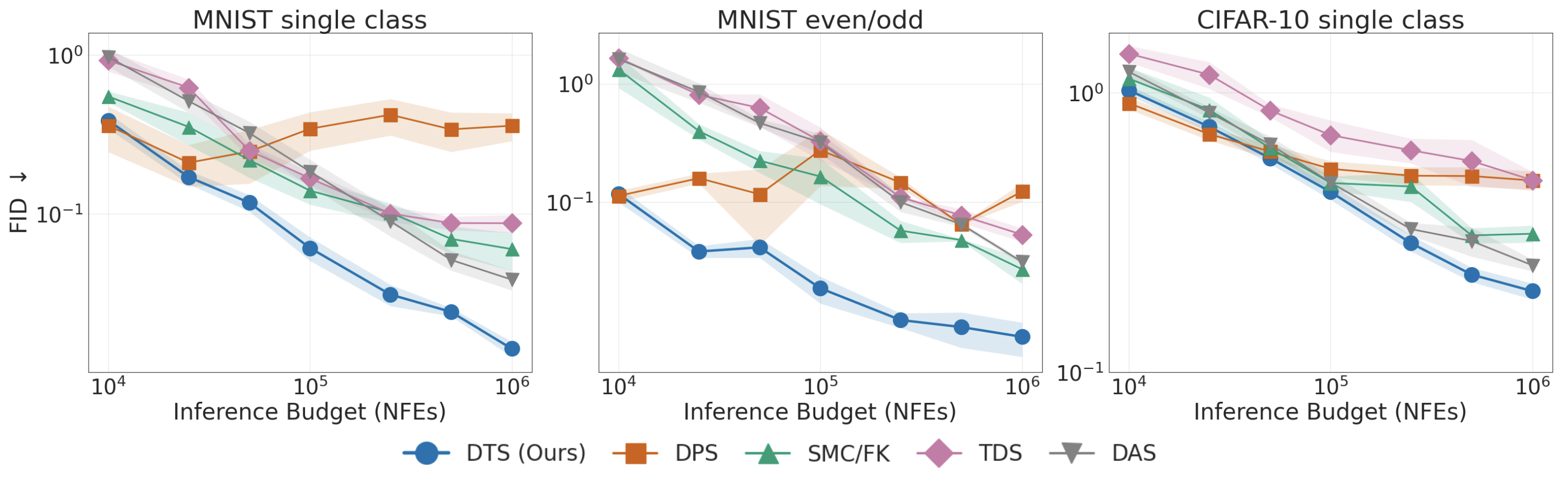}
    \caption{FID (lower is better) versus number of function evaluations for different methods on MNIST single digit generation averaged over all 10 digits (left), MNIST odd and even digit generation (center), and CIFAR-10 single class generation averaged over all 10 classes (right). All methods were evaluated with $5000$ generated samples per class.}
    \vspace{-1.25em}
    \label{fig:img_fid}
\end{figure*}

\paragraph{Baselines.} 
We compare the performance of \dtsa with DPS \citep{chung2023dps}, SMC/FK \citep{singhal2025general}, TDS \citep{wu2023practical} and DAS \citep{kim2025test}. DPS, TDS, and DAS use reward gradients, while SMC/FK is derivative-free. We report two distribution-based metrics -- Fr\'echet Inception Distance (FID) and CLIP maximum mean discrepancy (CMMD) \citep{jayasumana2024rethinking} -- that compare generated samples with ground truth samples from the dataset, in addition to average rewards and CLIP diversity (pairwise cosine distance).

\paragraph{Results.} \Cref{tab:img_results} reports the mean{\scriptsize$\pm$std} of various metrics for different methods after $10^6$ NFEs. In all three settings, \dtsa achieves the lowest FID and CMMD by a considerable margin, indicating it closely matches the true posterior. We observe that the margin of improvement on CIFAR-10 narrows slightly. We attribute this to reward noise: the CIFAR-10 classifier achieves an accuracy of $\sim\!85\%$, so its logits provide a noisier signal than the near-perfect classifier used on MNIST. Even so, \dtsa still outperforms all baselines. TDS and SMC in particular show characteristics of mode collapse with very high average rewards and low diversity, whereas DPS often generates samples that lie outside the support of the base model. \Cref{fig:img_fid} shows that \dtsa achieves very low FID across different NFEs and has better scaling properties with more compute compared to existing methods. \Cref{fig:img_samples} presents example outputs from each method, highlighting their specific characteristics. We present samples for all classes and plots of additional metrics as a function of NFEs in \Cref{app:exp_img}.

\begin{figure*}[b]
    \centering
    \vspace{-1em}
    \includegraphics[width=0.78\textwidth,trim={2em 0 2em 2em},clip]{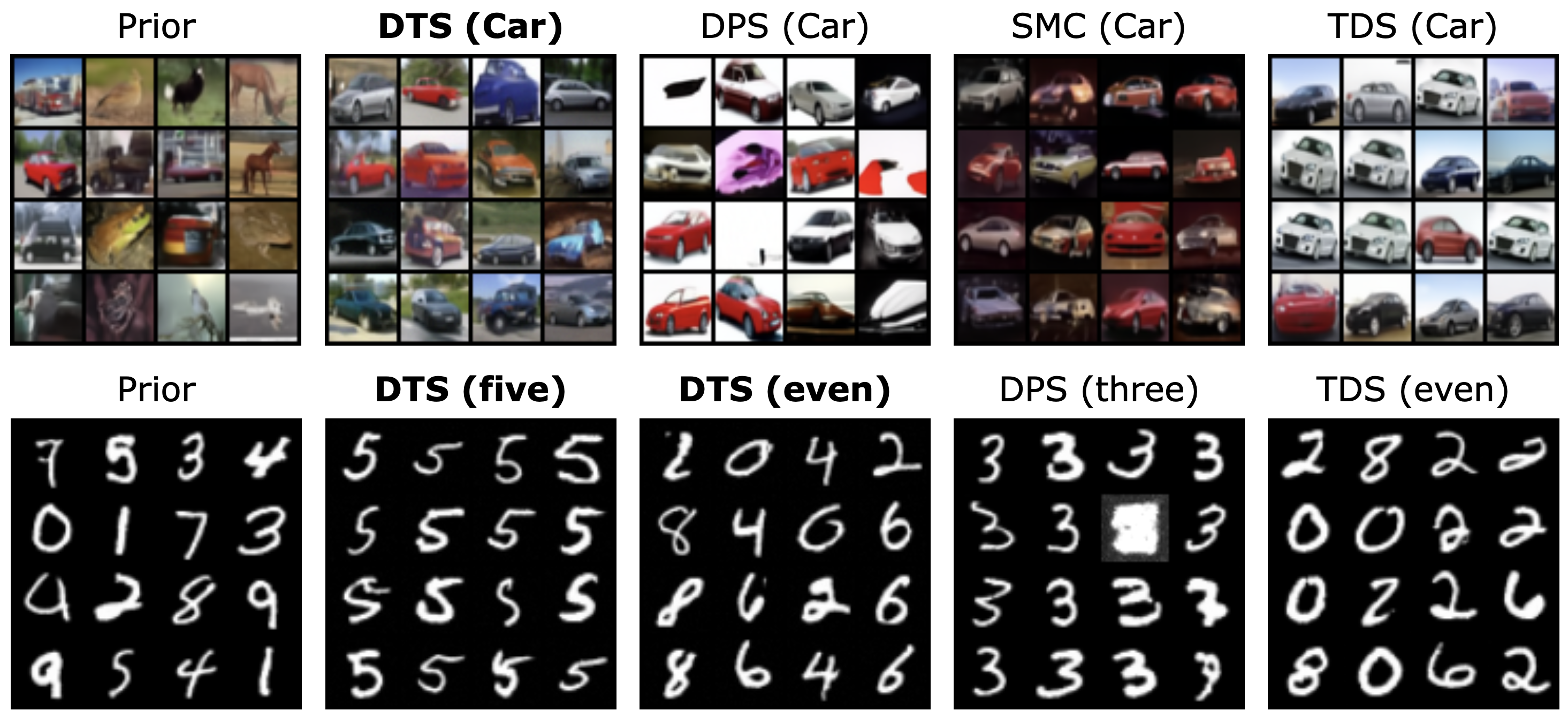}
    \caption{Samples generated from the CIFAR-10 (top) and MNIST (bottom) base diffusion models, and posterior samples using different methods at $10^6$ NFEs. Gradient-based guidance such as DPS can be unstable, leading to samples that lie outside the support of the prior. SMC-based methods struggle to accurately sample from multi-modal distributions -- for MNIST even digits, TDS oversamples from the digit two and undersamples from the digit four, and for CIFAR-10 car, both SMC and TDS suffer from mode collapse.}
    \vspace{-1em}
    \label{fig:img_samples}
\end{figure*}

\subsection{Text-to-image generation}
\label{sec:sd_exp}


\begin{figure}[t]
    \centering
    \includegraphics[width=\linewidth]{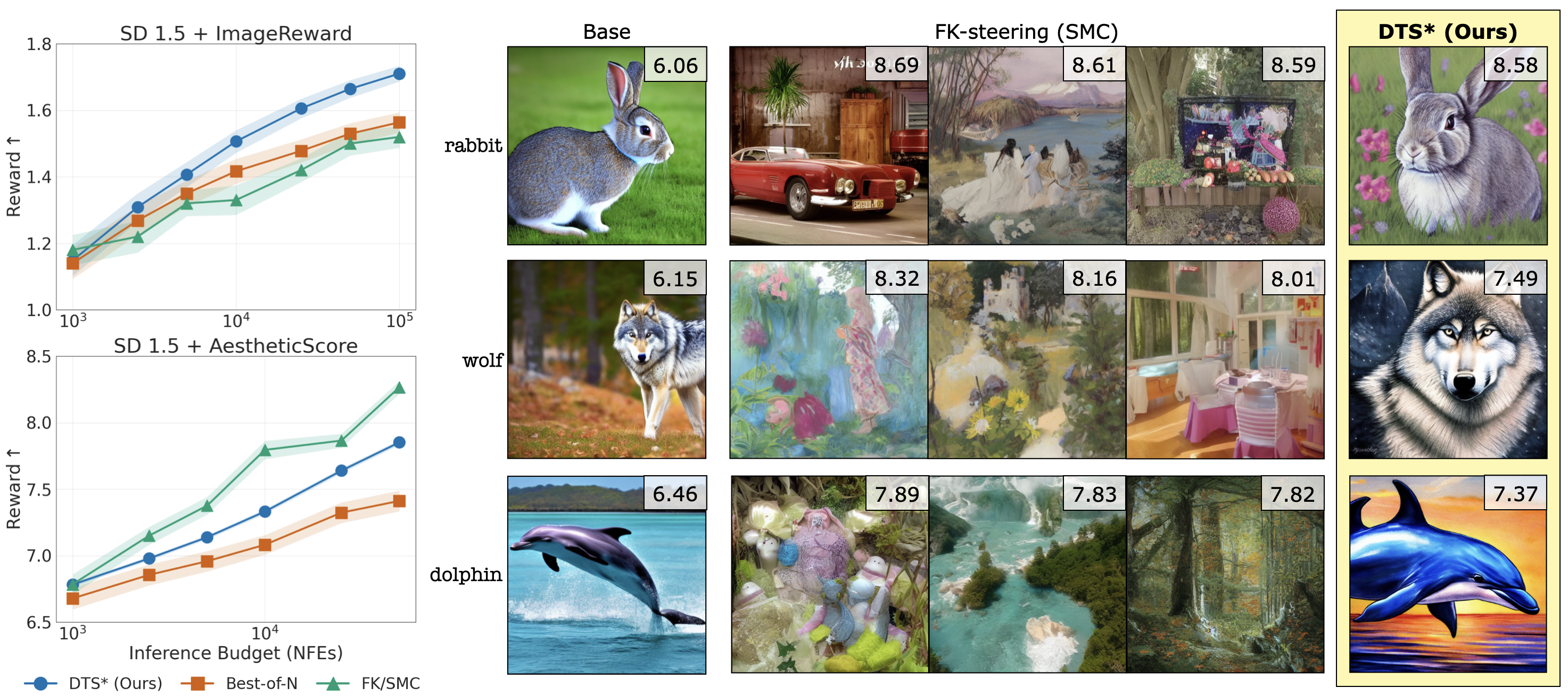}
    \caption{\textbf{Left:} Maximum ImageReward \citep{xu2023imagereward} vs. compute (NFEs) per prompt, averaged over $200$ prompts from DrawBench \citep{saharia2022photorealistic}, and maximum aesthetic score \citep{schuhmann2022laion} vs. compute (NFEs) per prompt, averaged over $45$ common animal prompts. \textbf{Right:} Samples generated using SD-v1.5 \citep{rombach2022high} for simple animal prompts and aesthetic score as the reward at $100k$ NFEs. For each prompt, \dtse faithfully matches the prompt while achieving high reward, whereas SMC samples score higher but visibly over-optimize. Numbers in the corner show aesthetic scores.}
    \vspace{-1.5em}
    \label{fig:sd_aesthetic}
\end{figure}

Large-scale diffusion models for text-conditioned image generation \citep{rombach2022high} can generate high-quality images and effectively generalize to new prompts. However, they struggle with complex captions, such as those that involve multiple components ("an emoji of a baby panda wearing a red hat, blue gloves, green shirt, blue pants") or unseen settings ("a panda making latte art"). A common approach to improve the quality of such complex prompts is to use human preferences \citep{xu2023imagereward} for fine-tuning the model or for inference-time alignment. The goal in this setting is to generate images with high reward while staying within the support of the base model.

\paragraph{Setting.} We use Stable Diffusion v1.5 \citep{rombach2022high}, a latent diffusion model, as the prior over $512\times 512$ images $\bx \sim p_\theta(\bx \mid \by)$ where $\by$ denotes the text prompt. We evaluate on two different settings: (a) DrawBench \citep{saharia2022photorealistic}, which is comprehensive benchmark of $200$ prompts, with ImageReward \citep{xu2023imagereward} $r(\bx, \by)$ that encodes prompt accuracy as well as human preferences; and (b) following \citet{black2024training}, we use $45$ common animals from the ImageNet-1000 dataset as prompts, with the  LAION aesthetics predictor \citep{schuhmann2022laion} $r(\bx)$ that encodes aesthetic quality of an image but does not check for prompt accuracy.

\paragraph{Baselines.} A strong baseline for high‑reward generation is \emph{best‑of‑$N$}, which draws $N$ samples from the base model and keeps the one with the highest reward. SMC  has also been applied to this problem 
\citep{singhal2025general}, but (a) as discussed in \Cref{sec:value}, it relies on one‑step value estimates that become inaccurate at high noise, and (b) \Cref{sec:img_exp} shows that it often collapses onto narrow
modes.
We compare \dtse with best-of-N and FK-Steering (SMC) \citep{singhal2025general}. For fair comparison, we set the number of candidates for best-of-N or the number of particles for FK-Steering based on the number of function evaluations of the base diffusion model, and record the highest reward for each method per prompt. 

\paragraph{Results.} 
\Cref{fig:sd_aesthetic} plots the maximum ImageReward and aesthetic score versus inference compute. \dtse outperforms the baselines for DrawBench prompts, see \Cref{fig:main_fig} for examples. One strong feature of \dtse is its favorable scaling properties with more compute. In the aesthetic score setting, FK/SMC achieves the highest rewards across NFEs, but we see severe over-optimization on all prompts. \dtse manages to strike the right balance between achieving high rewards while still maintaining faithfulness to the prior model. We hypothesize that since \dtse\ backs up \emph{soft} values, each node aggregates \emph{posterior mass} rather than peak density. Consequently, a reward spike that lies in a vanishing‑probability region of the prior contributes negligibly to the value estimates, resulting in an implicit KL‑regularization effect (cf. \Cref{sec:prelim_3}).

\subsection{Text generation}

\paragraph{Setting.}
We evaluate \dtse on text generation using MDLM \citep{sahoo2024simpleeffectivemaskeddiffusion}, a discrete diffusion language model. We generate three text completions of length 64 for each of 15 prompts introduced by \citet{han2023ssdlmsemiautoregressivesimplexbaseddiffusion}. The reward is defined as the log probability that the text is classified grammatically `acceptable' by a BERT-based classifier \citep{morris2020textattackframeworkadversarialattacks} trained on the Corpus of Linguistic Acceptability (CoLA) \citep{warstadt-etal-2019-neural}. 
We also report diversity by computing the number of distinct trigrams in each generated sequence. For decoding, we find that using \dtse with max-backup ($\lambda \rightarrow \infty$) yields the best performance. We compare our method against two baselines: FK/SMC \citep{singhal2025general} and best-of-$N$.

\paragraph{Results.}
As shown in \Cref{fig:text_generation}, \dtse consistently achieves the highest rewards as the number of function evaluations (NFEs) increases. Notably, reward functions in text domains are particularly susceptible to over-optimization, leading to the less diverse outputs observed when using FK/SMC. By contrast, \dtse produces outputs that have both high rewards and high diversity.
\begin{figure}[t]
    \centering
    \includegraphics[width=0.8\linewidth]{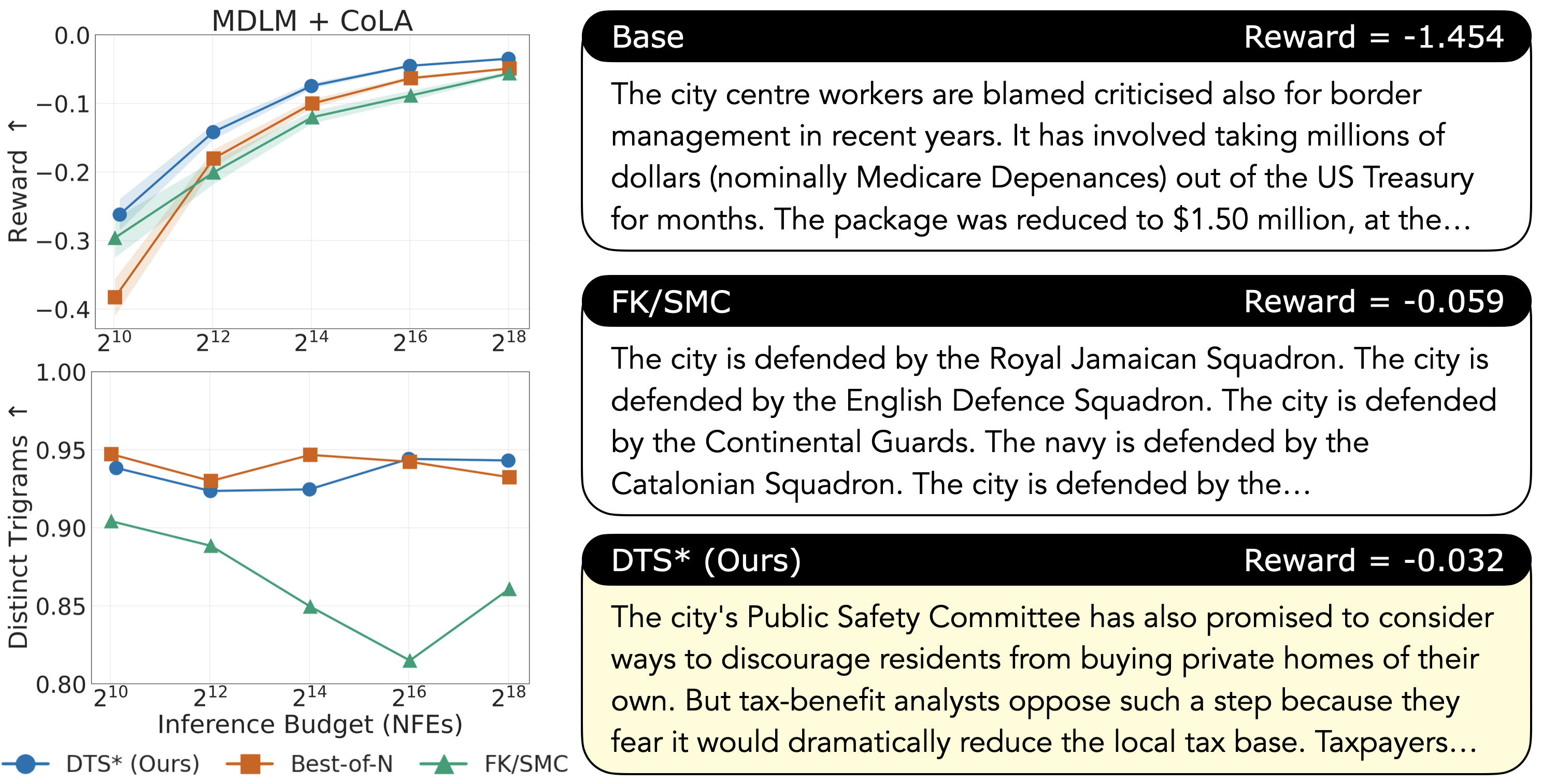}
    \caption{\textbf{Left:} Maximum reward and distinct trigrams vs. compute (NFEs) per prompt, across $15$ simple prompts with MDLM \citep{sahoo2024simpleeffectivemaskeddiffusion} using a classifier trained on CoLA \citep{morris2020textattackframeworkadversarialattacks} as the reward. \dtse obtains the highest reward while maintaining diverse outputs. \textbf{Right:} Typical samples generated by the base model, FK/SMC, and \dtse for the prompt \texttt{``The city''} at $2^{18}$ NFEs.}
    \vspace{-1.1em}
    \label{fig:text_generation}
\end{figure}

\section{Discussion}
\label{sec:discussion}

We have introduced a novel framework that casts inference–time alignment of diffusion models as a finite-horizon tree search.  By propagating terminal rewards via a soft value backup, our approach achieves global credit assignment and improved sample quality as compute increases.  Below, we highlight practical considerations, point out limitations, and suggest directions for further work.

\paragraph{High-dimensions and the role of pretrained model.} 
In high dimensions, an uninformed search tree grows exponentially with dimension, rendering pure tree search infeasible. A good quality pretrained model acts as a powerful prior, significantly pruning the effective search space. 

\paragraph{Learning the value function.}
In several applications of MCTS in game play, such as  AlphaZero \citep{silver2018general}, deep neural networks approximate both policy and value, speeding up search via learned rollouts and leaf estimates.  While our current work focuses on zero-shot inference-time alignment for \emph{any} unseen reward, an exciting future direction would be to integrate a learned value network for a fixed reward. 
This direction can address another important limitation -- when the number of samples is significantly larger than the tree, we see repeated samples corresponding to high probability leaves.

\paragraph{Compute cost.}
The control flow of tree-based methods is sequential, which makes them less parallelizable than particle-based methods such as SMC. However, as discussed in \Cref{sec:dts_design}, we implement an efficient version by batching calls to the diffusion model whenever possible, achieving per-step parallelism. On the extreme end, AlphaZero \citep{silver2018general} scaled MCTS across thousands of GPUs using multiple asynchronous `actors' to perform rollouts that update shared parameters; \dtsa can benefit from a similar design. 
Moreover, once the tree is constructed, sampling is purely pointer-chasing and incurs no further model calls, making repeated draws effectively free. 

\begin{ack}
This research is in part supported by CIFAR AI Chairs and the NSERC Discovery program. Mila and the Digital Research Alliance of Canada provided computational resources. 

\end{ack}

\clearpage

\bibliographystyle{plainnat}
\bibliography{references}

\clearpage

\appendix

\section{Extended related work}
\label{app:related}

We discussed the main approaches that have been proposed for inference-time alignment of diffusion models in \Cref{sec:pitfalls}. Below, we briefly review three tangentially related areas: fine‐tuning of diffusion models, their use in reinforcement learning, and entropy-regularized variations of Monte Carlo Tree Search.
 
\paragraph{Fine-tuning of diffusion models.}
 To sample from the target distribution $\pi^\ast$ for a fixed, known reward function, one option is to amortize the posterior sampling problem and update the model parameters via fine-tuning. The paradigm mirrors the trajectory of large–language-model alignment \citep{ziegler2019fine,rafailov2023direct}. Supervised preference finetuning trains directly on synthetic pairs scored by a reward model \citep{lee2023aligning,wu2023human}. Some early methods exploit differentiable objectives to back-propagate a single scalar all the way to the noise prediction network \citep{clark2024directly,prabhudesai2023aligning}, whereas more traditional reinforcement learning approaches cast each reverse step as an action and optimize expected reward \citep{black2024training}. To avoid over-optimization of the reward, recent works use KL regularization \citep{fan2023dpok, uehara2024fine, venkatraman2024amortizing}.

\paragraph{Diffusion models in reinforcement learning.}
Since the introduction of diffusion models as powerful frameworks for generative modeling, they have become popular for sampling actions or future states in RL. The earliest successes were in offline imitation learning, where some approaches model trajectories \citep{janner2022planning,ajay2023conditional} or expert policies \citep{chi2023diffusion} from offline datasets. Other works maximize a Q-function in addition to behavior cloning \citep{wang2023diffusion,kang2023efficient}, employ an explicit actor–critic scheme \citep{hansen2023idql}, or treat the critic as an energy function to guide the denoiser \citep{lu2023contrastive}. Some goal-conditioned extensions have also been proposed \citep{reuss2023goal,jain2024learning}. Recent works have explored similar ideas in the online setting \citep{yang2023policy,psenka2024learning,jain2024sampling}. Those methods aim to maximize return for control tasks, while we aim to draw unbiased samples from the reward-tilted distribution for any chosen reward.
 
\paragraph{Entropy-regularized MCTS.}
Monte-Carlo Tree-Search (MCTS) has recently been extended to soft-value objectives that incorporate an entropy bonus \citep{xiao2019maximum}, which uses a log-sum-exp value update and samples actions from a Boltzmann distribution, guaranteeing improved exploration at the cost of converging to the soft rather than the standard optimum. Follow-up work proposed to adapt the entropy term to a predefined value \citep{kozakowski2022planning} and decay the entropy term \citep{painter2023monte}. Very recently, \citet{morozov2024improving} used soft-backup MCTS to improve planning in Generative Flow Networks \citep{bengio2021flow}. Our Diffusion Tree Sampling (\dtsa\!) follows the same Boltzmann selection and soft value backup pattern, it is the first to embed a \emph{pre-trained diffusion kernel} inside the tree and to prove consistency for sampling from the KL-regularised posterior, not just selecting a single high-reward action. In this sense, \dtsa bridges the gap between entropy-regularized MCTS used for control and unbiased posterior sampling required for inference-time alignment of generative models.

\section{Sequential Monte Carlo for diffusion sampling}
\label{app:smc_background}

Many existing methods for inference-time diffusion alignment \citep{wu2023practical,trippe2023diffusion,cardoso2024monte,dou2024diffusion,kim2025test} apply sequential Monte Carlo (SMC) \citep{del2006sequential} to the reverse diffusion chain.
SMC maintains a population of $K$ particles to approximately sample from a sequence of intermediate targets $\{\pi_t(\bx_{t:T})\}_{t=T}^{0}$, culminating in the desired $\pi^{*}(\bx_0) \propto p_\theta(\bx_0)\exp(\lambda r(\bx_0))$. In diffusion alignment, one usually sets
\begin{align}
    \pi_t(\bx_{t:T}) \propto p(\bx_T)\prod_{s=t+1}^{T}p_\theta(\bx_{s-1}\mid\bx_s)\exp \left(\lambda\,\hat v_t(\bx_t)\right),
\end{align}
where $\hat v_t$ is a \emph{potential} approximating the soft value $V_t$. Each SMC iteration for $t = T, T-1, \ldots,0$ has three steps:
\begin{enumerate}[leftmargin=1.8em,itemsep=0.3em]
    \item \textbf{Propagation.}\; Sample particles $\tilde{\bx}_{t-1}^{(k)}\sim q_t(\,\cdot \mid \bx_t^{(k)})$, for $k = 1,\ldots,K$ where $q_t$ is the proposal distribution, often set to be the diffusion transition $p_\theta(\cdot\mid\bx_t)$.
    \item \textbf{Weighting.}\; Assign importance weights
      \begin{align}
          w_{t-1}^{(k)} = \underbrace{\frac{p_\theta(\tilde{\bx}_{t-1}^{(k)}\mid\bx_t^{(k)})}
                           {q_t(\tilde{\bx}_{t-1}^{(k)}\mid\bx_t^{(k)})}}_{\text{importance ratio}} \times \exp\left(\lambda\,\hat v_{t-1}(\tilde{\bx}_{t-1}^{(k)})\right).
        \label{eq:smc_weights}
      \end{align}
      The first factor corrects for using a proposal and the second tilts weights toward high estimated value.
    \item \textbf{Resampling.}\; Resample $\{\tilde{\bx}_{t-1}^{(k)}\}_{k=1}^K$ proportional to $\{w_{t-1}^{(k)}\}_{k=1}^K$ to obtain an equally weighted particle set $\{\bx_{t-1}^{(k)}\}_{k=1}^{K}$ for the next iteration.
\end{enumerate}

Classical SMC guarantees that, as $K\!\to\!\infty$ and if the potentials are exact, the empirical measure $\sum_k w_{0}^{(k)}\delta\left({\bx_{0}^{(k)}}\right)$ converges to the target distribution $\pi^{*}$, where $\delta(x)$ is the Dirac delta at $x$. In practice, however, this repeated sampling procedure can reduce the diversity of samples, especially when the weights have high variance. This results in an \emph{effective sample size} which is much lower than $K$. 

Another major issue when applying SMC to diffusion models is that estimating the soft value function $V_t$ is not straightforward and errors in the approximation can lead to inaccurate sampling. The next subsection discusses the \emph{value‑estimation problems} in more detail.

\section{Connection with Generative Flow Networks}
\label{app:gflownet}

Diffusion Tree Sampling can be viewed as an on-the-fly, non-parametric realization of the ideas behind Generative Flow Network (GFlowNet) \citep{bengio2021flow}. Both frameworks ultimately seek to sample from an unnormalised density:
\begin{align*}
   \pi^\ast(x)=\frac{1}{\mathcal{Z}}\,f(x), \quad \mathcal{Z}=\int f(x)\,dx,
\end{align*}
but they do so with different machinery and at different points in the learning–inference pipeline.

GFlowNets define a probability over complete paths $\tau=(\bs_0\!\to\!\cdots\!\to\!\bs_T=x)$ through
\begin{align*}
   P_\theta(\tau)=\prod_{t=1}^{T} P_\theta(\bs_t\mid \bs_{t-1}),
\end{align*}
and train the parameters~$\theta$ so that the \emph{forward flow} leaving every non-terminal state equals the \emph{backward flow} entering it plus injected terminal reward $r(\bx)= \log f(\bx)$. For the special case of a tree-structured graph, this constraint in log form is a soft Bellman equation \citep{tiapkin2024generative, deleu2024discrete, mohammadpour2024maximum}:
\begin{align*}
   F(\bs)=\frac1\lambda \log\sum_{s'\in\text{Child}(s)} P_\theta(\bs'\mid \bs)\,\exp \left(\lambda F(\bs')\right),
\end{align*}
with $F(\bs)$ the learned log-flow function.

\dtsa satisfies the same soft Bellman recursion (cf. \Cref{eq:soft_bellman}), but does so \emph{without} learning parameters. During tree construction, \dtsa estimates the soft value  $V_t$ by Monte-Carlo log-sum-exp backups; selection then samples children proportionally to $\exp(\lambda\hat{v}_{t-1})$, where $\hat{v}_{t-1}$ is the estimated soft value. Repeated roll-outs make the empirical terminal distribution converge to the reward-tilted posterior $\pi^\ast(\bx_0)\propto p_\theta(\bx_0)\exp\left(\lambda r(\bx_0)\right)$, just as a perfectly trained GFlowNet would.

The key differences between \dtsa and GFlowNets are summarized below.
\begin{itemize}
  \item \textbf{Proposal.}  
        \dtsa uses a \emph{fixed}, pretrained diffusion kernel $p_\theta$ as a proposal, whereas GFlowNets learn their forward policy $P_\theta$.
  \item \textbf{Learning vs. search.}  
        \dtsa performs pure inference without updating any parameters, whereas GFlowNets learn the parameters of the sampler to amortize future sampling.
  \item \textbf{Computational regime.}  
        \dtsa excels when one has a strong prior and large \emph{inference} budget for new rewards; GFlowNets shine when the reward is fixed and repeated queries amortize the \emph{training} cost.
\end{itemize}

Because DTS is a search procedure, it is ideal for adapting a pretrained diffusion model to different unseen reward functions without retraining.  GFlowNets, in contrast, learn a fast parametric sampler for a single reward. 

\section{Proofs and derivations}
\label{app:proof}

\subsection{Derivation of Equations 5 and 6}
\label{app:proof_1}

We derive the recursive relation satisfied by the soft value function as well as the expression for the optimal policy in \Cref{sec:prelim} for completeness. 

\paragraph{Soft value function.} This recursive relation is analogous to the soft Bellman equation in maximum entropy RL \citep{ziebart2008maximum, haarnoja2017reinforcement}.
Starting from the definition of $V_t(\bx_t)$:
\begin{align*}
    V_t(\bx_t) &= \frac{1}{\lambda}\log \mathbb{E}_{p_\theta(\bx_{0:t-1}|\bx_t)}\left[\exp\left(\lambda r(\bx_0)\right)\right] \\
    &= \frac{1}{\lambda}\log \int p(\bx_0, \bx_1, \ldots, \bx_{t-1} | \bx_t) \exp\left(\lambda r(\bx_0)\right) d\bx_0 d\bx_1 \ldots d\bx_{t-1} \\
    &= \frac{1}{\lambda}\log \int p(\bx_0, \bx_1, \ldots \bx_{t-2}| \bx_{t-1})\,p(\bx_{t-1} | \bx_t) \exp\left(\lambda r(\bx_0)\right) d\bx_0 d\bx_1 \ldots d\bx_{t-1} \\
    &= \frac{1}{\lambda}\log \int p(\bx_{t-1} | \bx_t) \underbrace{\left( \int p(\bx_0, \bx_1, \ldots \bx_{t-2}| \bx_{t-1}) \exp\left(\lambda r(\bx_0)\right) d\bx_0 d\bx_1 \ldots d\bx_{t-2} \right)}_{= \exp\left(\lambda V(\bx_{t-1})\right)} d\bx_{t-1} \\
    &= \frac{1}{\lambda}\log \int p(\bx_{t-1} | \bx_t) \exp\left(\lambda V(\bx_{t-1})\right) d\bx_{t-1} = \frac{1}{\lambda}\log \mathbb{E}_{p(\bx_{t-1} | \bx_t)}\left[\exp\left(\lambda V(\bx_{t-1})\right)\right].
\end{align*}

The above relation combined with the terminal condition $V_0(\bx_0)=r(\bx_0)$ gives \Cref{eq:soft_bellman}.

\paragraph{Optimal policy.}
The joint target density over the full chain $(\bx_0,\ldots,\bx_{t-1},\bx_{t})$ is given by:
\begin{align*} 
\pi^\ast(\bx_0,\ldots,\bx_{t-1},\bx_{t}) &= \frac{1}{\mathcal{Z}}\;
p_\theta(\bx_0,\ldots,\bx_{t-1},\bx_{t})\,\exp\left(\lambda r(\bx_0)\right),
\end{align*}
where $\mathcal{Z}$ represent the normalization constant of this joint density.

The marginal joint density of $(\bx_t,\bx_{t-1})$ under $\pi^\ast$ is:
\begin{align*}
    \pi^\ast(\bx_{t},\bx_{t-1}) &= \frac{1}{\mathcal{Z}}\int p_\theta(\bx_0,\ldots,\bx_{t-1},\bx_{t}) \exp\left(\lambda r(\bx_0)\right) d\bx_0 \ldots d\bx_{t-2} \\
    &= \frac{1}{\mathcal{Z}}p_\theta(\bx_t) p_\theta(\bx_{t-1}\mid \bx_t) \left(\int p_\theta(\bx_0,\ldots,\bx_{t-2} \mid \bx_{t-1}) \exp\left(\lambda r(\bx_0)\right) d\bx_0 \ldots d\bx_{t-2}\right)\\
    &= \frac{1}{\mathcal{Z}}p_\theta(\bx_t) p_\theta(\bx_{t-1}\mid \bx_t) \exp\left(\lambda V(\bx_{t-1})\right)
\end{align*}
Similarly, the marginal density of $\bx_t$ under $\pi^\ast$ is:
\begin{align*}
    \pi^\ast(\bx_{t}) = \frac{1}{\mathcal{Z}}\; p_\theta(\bx_t) \exp\left(\lambda V(\bx_{t})\right)
\end{align*}
By dividing these two marginals, we get the transitions under the optimal policy:
\begin{align*}
    \pi^\ast(\bx_{t-1}\mid\bx_t) &= \frac{\pi^\ast(\bx_t,\bx_{t-1})}{\pi^\ast(\bx_t)}
    = \frac{p_\theta(\bx_{t-1}\mid \bx_t) \exp\left(\lambda V(\bx_{t-1})\right)}{\exp\left(\lambda V(\bx_{t})\right)} \\
    &= \frac{p_\theta(\bx_{t-1}\mid\bx_t)\exp\left(\lambda V_{t-1}(\bx_{t-1})\right)}{\int p_\theta(\bx_{t-1}\mid\bx_t)\exp\left(\lambda V_{t-1}(\bx_{t-1})\right) d\bx_{t-1}}.
\end{align*}
The above relation gives the optimal policy from \Cref{eq:optimal_policy}.

\subsection{Proof of Proposition 1}
\label{app:proof_2}

\addtocounter{proposition}{-1}
\begin{proposition}[Asymptotic consistency]
\label{prop:mcdt_proof_app}
Let $r$ be bounded and $\lambda > 0$, then \dtsa produces a sequence of terminal states whose empirical distribution converges to the optimal policy $\pi^\ast$ as the number of tree iterations $M\to\infty$.
\end{proposition}

\begin{proof}
We use $p(\cdot \mid \bx_t)$ to denote a general proposal distribution. For application to diffusion alignment, this would correspond to transitions under the pretrained model $p_\theta(\cdot \mid \bx_t)$. Additionally, we use $\hat{q}(\cdot \mid \bx_t)$ to denote the transition density of \dtsa.
\begin{description}[leftmargin=0em,itemsep=0.3em,labelsep=0em]
\item \textbf{Step 1: Transition probability under \dtsa.}
Recall that under \dtsa, given a node \(\bx_t\), we create each child by sampling from the base model $p(\cdot \mid \bx_t)$. During tree traversal, we select the next state $\bx_{t-1}$ proportional to the exponentiated soft value function. Thus, the transition probability of \dtsa from $\bx_t$ to $\bx_{t-1}$ is given by:
\begin{align}
    \hat{q}(\bx_{t-1} \mid \bx_t) = \frac{p(\bx_{t-1} \mid \bx_t) \exp\left(\lambda \hat{v}(\bx_{t-1})\right)}{\int p(\bx_{t-1} \mid \bx_t)\exp\left(\lambda \hat{v}(\bx_{t-1})\right) \,d\bx_{t-1}} = \frac{p(\bx_{t-1} \mid \bx_t) \exp\left(\lambda \hat{v}(\bx_{t-1})\right)}{\exp\left(\lambda \hat{v}(\bx_t)\right)},
\end{align}
where the second equality follows from the definition of the soft Bellman equation:
\begin{align*}
    \hat{v}(\bx_t) = \frac{1}{\lambda} \log \mathbb{E}_{\bx_{t-1}\sim p(\cdot \mid \bx_t)}\left[\exp\left(\lambda \hat{v}(\bx_{t-1})\right)\right].
\end{align*}
\item \textbf{Step 2: Joint density of trajectory.}
Recall that the root node of \dtsa contains a dummy state $\bx_{T+1}$ that transitions to the diffusion process prior $\hat{q}(\bx_T \mid \bx_{T+1}) = \mathcal{N}(0,I)$. Then, the joint density of a full trajectory $ \{\bx_T, \bx_{T-1}, \dots, \bx_0\}$ under \dtsa is given by:
\begin{align*}
    \hat{q}(\bx_{T},\bx_{T-1},\ldots,\bx_0) &= \prod_{t=1}^{T+1} \hat{q}(\bx_{t-1} \mid \bx_t) = \prod_{t=1}^{T+1} \frac{p(\bx_{t-1} \mid \bx_t)\exp\left(\lambda \hat{v}(\bx_{t-1})\right)}{\exp\left(\lambda \hat{v}(\bx_t)\right)}\\
    &= \frac{\exp\left(\lambda \hat{v}(\bx_0)\right)}{\exp\left(\lambda \hat{v}(\bx_{T+1})\right)}\prod_{t=1}^{T+1}p(\bx_{t-1} \mid \bx_t) \\
    &= \frac{\exp\left(\lambda \hat{v}(\bx_0)\right)}{\exp\left(\lambda \hat{v}(\bx_{T+1})\right)}p(\bx_{T},\bx_{T-1},\ldots,\bx_0).
\end{align*}
\item \textbf{Step 3: Marginalizing.}
Marginalizing over intermediate states \(\bx_1, \dots, \bx_T\), we get the distribution of terminal state \(\bx_0\):
\begin{align*}
    \hat{q}(\bx_0) &= \int \hat{q}(\bx_{T},\bx_{T-1},\ldots,\bx_0) \,d\bx_T d\bx_{T-1}\ldots d\bx_1 \\
    &= \frac{\exp\left(\lambda \hat{v}(\bx_0)\right)}{\exp\left(\lambda \hat{v}(\bx_{T+1})\right)}\int p(\bx_{T},\bx_{T-1},\ldots,\bx_0) \,d\bx_T d\bx_{T-1}\ldots d\bx_1\\
    &= \frac{\exp\left(\lambda \hat{v}(\bx_0)\right)}{\exp\left(\lambda \hat{v}(\bx_{T+1})\right)} p(\bx_0).
\end{align*}
By definition, the soft value function at the terminal node is $\hat{v}(\bx_0)=r(\bx_0)$. Plugging this and using the definition of value function from \Cref{eq:soft_value}, we have:
\begin{align*}
    \hat{q}(\bx_0) &= \frac{\exp\left(\lambda r(\bx_0)\right)p(\bx_0)}{\int p(\bx_{T},\bx_{T-1},\ldots,\bx_0 \mid \bx_{T+1}) \exp\left(\lambda r(\bx_{0})\right) \,d\bx_T d\bx_{T-1}\ldots d\bx_1 d\bx_0} \\
    &= \frac{\exp\left(\lambda r(\bx_0)\right)p(\bx_0)}{\int p(\bx_0) \exp\left(\lambda r(\bx_{0})\right) \,d\bx_0}.
\end{align*}
This has the form of the target distribution in \Cref{eq:optimal_policy}, except that it uses the value estimates $\hat{v}$ that are calculated based on rollouts starting from each state $\bx_t$. In the limit of infinite rollouts, these value estimates approach the true soft values, confirming that the sampling distribution \(\hat{q}\) from \dtsa exactly matches the target distribution $\pi^*$. 

Therefore, \dtsa is consistent, as it correctly generates samples from the desired target distribution asymptotically.
\end{description}
\end{proof}

\section{DTS and $\text{DTS}^\star$ algorithm}
\label{app:algo}

\begin{algorithm}[H]
\caption{Diffusion Tree Sampling (\dtsa\!) and Diffusion Tree Search (\dtse\!)}
\begin{algorithmic}[1]
\STATE {\bf Input:} base policy $p_\theta$, reward function $r$, number of iterations $M$, inverse temperature $\lambda$, parameters $C, \alpha, c_\text{uct}$
\STATE Initialize root node $\bx_{T+1}$ with dummy value, $\hat{v}(\bx_{T+1}) = 0$, $N(\bx_{T+1})=1$
\STATE Initialize tree $\mathcal{T}$ with root node $\bx_{T+1}$
\FOR{$m=1,\dots,M$}
\STATE $P \leftarrow \{\bx_{T+1}\}$
\STATE Set $t \leftarrow {T+1}$
\STATE \texttt{\textcolor{OliveGreen}{// Selection}}
\WHILE{$|\mathcal{C}(\bx_t)| \geq C \cdot N(\bx_t)^\alpha$ and $t>0$}
\STATE {\color{NavyBlue}\textbf{[\dtsa\!]}} select child probabilistically:
$
\bx_{t-1} \sim \frac{\exp(\lambda \hat{v}(\bx_{t-1}))}{\sum_{\bx'\in\mathcal{C}(\bx_t)} \exp(\lambda \hat{v}(\bx'))}
$
\STATE {\color{BrickRed}\textbf{[\dtse\!]}} select child maximizing UCT:
$
\bx_{t-1} = \arg\max_{\bx'\in\mathcal{C}(\bx_t)} \hat{v}(\bx') + c_\text{uct}\sqrt{\frac{\log N(\bx_t)}{N(\bx')}}
$
\STATE $P \leftarrow P \cup \{\bx_{t-1}\}$
\STATE $t \leftarrow t-1$
\ENDWHILE
\STATE \texttt{\textcolor{OliveGreen}{// Expansion: \!\!\!\!expand $\bx_t$ by sampling a new child}}
\IF{$t>0$, and $|\mathcal{C}(\bx_t)| < C \cdot N(\bx_t)^\alpha$}
\STATE \texttt{\textcolor{OliveGreen}{// Rollout: \!\!\!\!from new node $\bx_{t-1}$ sample rollout path to terminal $\bx_0$}}
\WHILE{$t>0$}
\STATE $\bx_{t-1} \sim p_\theta(\cdot\mid \bx_t),\quad \hat{v}(\bx_{t-1})=0,\quad N(\bx_{t-1})=1$
\STATE $P \leftarrow P \cup \{\bx_{t-1}\}$
\STATE $t \leftarrow t-1$
\ENDWHILE
\ENDIF
\STATE Evaluate terminal reward: $\hat{v}(\bx_0)=r(\bx_0)$
\STATE \texttt{\textcolor{OliveGreen}{// Backup: \!\!\!\!update value along path $P$}}
\FOR{$t=0,\dots,T$}
\STATE Soft backup:
$
\hat{v}(\bx_{t+1}) \leftarrow \frac{1}{\lambda}\log\sum_{\bx_t\in\mathcal{C}(\bx_{t+1})}\exp(\lambda\hat{v}(\bx_t))
$
\STATE Update visits: $N(\bx_{t+1}) \leftarrow N(\bx_{t+1}) + 1$
\ENDFOR
\ENDFOR
\STATE \textbf{return} $\mathcal{T}$
\end{algorithmic}
\end{algorithm}

\section{Implementation details for DTS and $\text{DTS}^\star$}
\label{app:implementation}

\subsection{Tree structure}
\label{app:implementation_tree}

The algorithm presented in \Cref{sec:dts_sampling} and \Cref{app:algo} allows every state $\bx_t$ along the denoising trajectory to be considered for branching. However, in practice, we only branch every few timesteps. We noticed very little difference in performance between the two cases for the same number of function evaluations, however, we expect branching at every step to outperform for a very high compute budget. We match the tree branching schedule with the resampling schedule for all baselines with SMC, similar to the setting from \citet{singhal2025general}. The exact setting for each experiment is presented in \Cref{tab:steps}, where we always branch at the root node corresponding to $t = T$.

\begin{table}[h!]
\caption{Branching schedule for \dtsa and \dtse, which is also the resampling schedule used for SMC-based methods -- SMC/FK \citep{singhal2025general}, TDS \citep{wu2023practical}, DAS \citep{kim2025test}.}
\label{tab:steps}
\centering
\begin{tabular}{%
  l c c
}
\toprule
Domain & Total denoising steps  &  Branching schedule\\
\midrule
        Two-dimensional & $100$ & $100 (\operatorname{root}), 80, 60, 40, 20$ \\
        Image pixels (MNIST, CIFAR-10) & $50$ & $50 (\operatorname{root}), 40, 30, 20, 10$ \\
        Image latents (SD-v1.5) & $50$ & $50 (\operatorname{root}), 40, 30, 20, 10$ \\
        Text tokens (MDLM) & $64$ & $64 (\operatorname{root}), 54, 44, 34, 24, 14$ \\ 
\bottomrule
\end{tabular}
\end{table}

Apart from this, we have hyperparameters associated with progressive widening that control the maximum number of branches at any node. We used $\alpha = 0.8$ and $C = 2$ for all two-dimensional and image experiments and $\alpha = 0.7$ and $C = 2$ for text generation. There is a scope of improving the performance of \dtsa and \dtse further by tuning these parameters for specific tasks.

\subsection{Experiment details}
\label{app:implementation_exp}

\begin{expblock}[]{Illustrative 2D}
    \begin{description}[leftmargin=*,itemsep=0.3em,labelsep=0em]
        \item \textbf{Base diffusion model:} The denoising network is an MLP that takes as input the 2-dimensional data $\bx_t$ and the timestep $t$ and outputs a 2-dimensional noise prediction. The timestep is transformed using sinusoidal embeddings \citep{vaswani2017attention}. The network has four hidden layers of $128$ dimension each with the sigmoid linear unit (SiLU, \citep{hendrycks2016gaussian}) activation. We used the linear noise schedule with $\beta_\textrm{min} = 0.001$ and $\beta_\textrm{max} = 0.07$ and the score matching objective. The optimizer used for training was Adam \citep{kingma2014adam} with a learning rate of $3\times 10^{-3}$. We train the model for $500$ epochs on a training set of $10000$ samples. 
        \item \textbf{Reward function:} \textit{Gaussian mixture:} The reward function is:
        \begin{align*}
            r(\bx)= \log\left(\displaystyle\sum_{i=1}^{8} w_i\,\exp\!\left(-\|\mathbf x-\boldsymbol{\mu}_i\|^{2}/2\sigma^2\right)\right),
        \end{align*}
        where $w_i = \exp(1.5\,i),\;\;\boldsymbol{\mu}_i = 4 \left(\cos\tfrac{2\pi(i-1)}{8},\;\sin\tfrac{2\pi(i-1)}{8}\right),\;\; i=1,\dots,8$, with $\sigma=0.3$.
        \textit{Checkerboard:} The reward is negative distance from the center $r(\bx) = -0.5 \|\bx\|^2$.
    \end{description}
\end{expblock} 

\begin{expblock}{Class-conditional MNIST}
    \begin{description}[leftmargin=*,itemsep=0.3em,labelsep=0em]
        \item \textbf{Base diffusion model:} The denoising network is a Unet architecture \citep{ronneberger2015u} that operates on images of size $32 \times 32 \times 1$ (upscaled from $28 \times 28 \times 1$) with block channels $\{32,64,128,256\}$. We use the DDIMScheduler\footnote{\url{https://huggingface.co/docs/diffusers/en/api/schedulers/ddim}} from diffusers library with default parameters, except we set $\eta = 1.0$ so the inference process is stochastic like DDPMs \citep{ho2020denoising}. We use the AdamW optimizer with a learning rate of $10^{-4}$ for $100$ epochs on the MNIST training set.
        \item \textbf{Reward function:} We train a classifier $p(c\mid \bx)$ on the MNIST training set. The classifier is a convolutional neural network \citep{lecun2015deep} using two $5\times 5$ kernels with $(16,32)$ channels followed by $2\times 2$ max pooling operation with ReLU activations. The features are then flattened and followed by a linear layer with $10$ outputs corresponding to the classes. The network was trained using Adam optimizer with learning rate $10^{-3}$. The reward function for single class generation is the log likelihood of the class $r_i(\bx) = \log p(c = i \mid \bx)$ for $i \in \{0,1,\ldots,9\}$. For the even or odd generation, it is defined as $r(\bx) = \max_{i \in \mathcal{S}} \log p(c = i \mid \bx)$, where $\mathcal{S} = \{0,2,4,6,8\}$ for even digit generation and $\mathcal{S} = \{1,3,5,7,9\}$ for odd digit generation.
    \end{description}
\end{expblock}

\begin{expblock}{Class-conditional CIFAR-10}
    \begin{description}[leftmargin=*,itemsep=0.3em,labelsep=0em]
        \item \textbf{Base diffusion model:} We used the pre-trained diffusion model \texttt{ddpm-cifar10-32}\footnote{\url{https://huggingface.co/google/ddpm-cifar10-32}} from Hugging Face, which uses a Unet architecture and diffuses over $32\times 32\times 3$ images in pixel-space. We use the DDIMScheduler with $\eta = 1.0$ for stochastic denoising.
        \item \textbf{Reward function:}  We train a classifier $p(c\mid \bx)$ on the CIFAR-10 training set. The classifier uses a ResNet‑18 \citep{he2016deep} backbone that outputs an embedding which is average pooled, flattened, and passed to a single linear layer with $10$ outputs. The network is trained using Adam optimizer with learning rate $10^{-3}$. Similar to MNIST single class generation, the reward function is the log likelihood of the class $r_i(\bx) = \log p(c = i \mid \bx)$ for $i \in \{0,1,\ldots,9\}$.
    \end{description}
\end{expblock} 

\begin{expblock}{Text-to-image}
    \begin{description}[leftmargin=*,itemsep=0.3em,labelsep=0em]
        \item \textbf{Base diffusion model:}  We use Stable Diffusion v1.5\footnote{\url{https://huggingface.co/stable-diffusion-v1-5/stable-diffusion-v1-5}} from Hugging Face, which is a latent diffusion model \citep{rombach2022high}. The diffusion process is defined over $64\times 64\times 4$ latent variables, which are obtained by encoding $512\times 512\times 3$ images using a variational autoencoder. The model uses CLIP \citep{radford2021learning} to encode text prompts into embeddings which are then used to condition the generative process via classifier-free guidance \citep{ho2021classifier}. We use the DDIMScheduler with $\eta = 1.0$.
        \item \textbf{Reward function:}  We use pre-trained models as reward functions including ImageReward\footnote{\url{https://github.com/THUDM/ImageReward}} $r(\bx, \by)$ that encodes prompt accuracy as well as human preferences the LAION aesthetic score predictor\footnote{\url{https://github.com/LAION-AI/aesthetic-predictor}} $r(\bx)$ that encodes aesthetic quality of an image.
    \end{description}
\end{expblock} 

\begin{expblock}{Conditional text}
    \begin{description}[leftmargin=*,itemsep=0.3em,labelsep=0em]
        \item \textbf{Base diffusion model:} We use MDLM\footnote{\url{https://huggingface.co/kuleshov-group/mdlm-owt}} for our text generation experiments. This is a discrete diffusion model with $110$M parameters that directly predicts the tokens. We define the diffusion process over a context length of $64$ tokens with $64$ sampling steps and use the standard discrete unmasking update for stochastic denoising.
        \item \textbf{Reward function:}  We use a BERT-based classifier\footnote{\url{https://huggingface.co/textattack/roberta-base-CoLA}} trained on the Corpus of Linguistic Accepatbility (CoLA) \citep{warstadt-etal-2019-neural}. This reward function $r(\bx)$ encodes the linguistic acceptability of a given string $\bx$. The reward is the log probability of the text being "acceptable". We find this model to be more robust to reward-hacking than alternatives. 
    \end{description}
\end{expblock} 

\subsection{Compute}
\label{app:implementation_compute}

We report execution times on a single A100 GPU with 80 gigabytes of memory.
\begin{itemize}[leftmargin=1.8em,itemsep=0.3em]
    \item Each 2D experiment including all methods runs in 15 minutes. Adding up the time over five seeds and two different datasets, the combined run time is approximately 2.5 GPU hours.
    \item The MNIST and CIFAR-10 class-conditional experiments use approximately 3 GPU hours per class including all methods. Over all 22 tasks (10 MNIST single digit + 2 MNIST even/odd + 10 CIFAR-10 classes) equals approximately 66 GPU hours.
    \item The text-to-image experiments using Stable Diffusion v1.5 require roughly 30 minutes per prompt across all methods. Adding up all 200 prompts from DrawBench and 45 animal prompts, reproducing all experiments requires approximately 123 GPU hours.
    \item The text generation experiments using MDLM requires roughly 30 minutes per prompt. Thus, generating 3 completions per prompt for the 15 prompts requires roughly 22.5 GPU hours.
\end{itemize}

\section{Details of baselines}
\label{app:baselines}

We re-implemented all baseline methods in our unified codebase since most of them use SMC as a backbone and share the same underlying infrastructure. Each implementation was validated by reproducing the quantitative results reported in its original paper. \Cref{app:smc_background} provides a concise primer on SMC for reference. The complete source code including all baselines will be released publicly upon publication of this work.

\paragraph{DPS.} Diffusion Posterior Sampling \citep{chung2023dps} was originally proposed for noisy inverse problems such as image super-resolution and de-blurring using the gradient of the final objective. To adapt this method for general reward functions, we make a minor modification by replacing the gradient of the inverse problem objective with the gradient of the reward function:
\begin{align}
    \tilde{\bx}_{t-1} \sim \mathcal{N}\left(\mu_{\theta}(\bx_{t},t)\,+\,\lambda\,\sigma_{t}^{2}\,\nabla_{\bx_t} r(\hat{\bx}_0(\bx_t)), \sigma_t^2\,I\right),
    \label{eq:dps}
\end{align}
where $\hat{\bx}_0$ is obtained using Tweedie's formula (cf. \Cref{sec:pitfalls}), $\mu_\theta$ is the predicted mean of the base diffusion model, and $r(\bx)$ is the reward function in the two-dimensional experiments and classifier log likelihoods $\log p(c = i \mid \bx)$ for class-conditional image experiments. The official implementation is provided \href{https://github.com/DPS2022/diffusion-posterior-sampling}{here}.

\paragraph{SMC/FK-Steering.} 
In our paper, SMC refers to the simplest variant, FK-Steering \citep{singhal2025general}, which defines different weighting schemes and uses the pre-trained diffusion model as the proposal distribution. As per the setting in \citet{singhal2025general}, we perform the resampling step at fixed intervals during denoising (given in \Cref{tab:steps}) and use adaptive resampling to increase diversity of generated samples. Our sampling experiments (two-dimensional and class-conditional image generation) use the \texttt{`diff'} potential with $\lambda = 1.0$, whereas the search experiments (text-to-image and text generation) use the \texttt{`max'} potential with $\lambda = 10.0$. The weights for resampling are given by \Cref{eq:smc_weights} where the proposal is equal to the pre-trained diffusion transition and the value estimates are equal to:
\begin{align*}
      \hat{v}_{t-1}^{\text{diff}}(\tilde{\bx}_{t-1}) &= r \left(\hat{\bx}_0(\tilde{\bx}_{t-1})\right) - r \left(\hat{\bx}_0(\bx_t)\right), \quad
      \hat{v}_{T}^{\text{diff}}(\bx_{T}) = r \left(\hat{\bx}_0(\bx_{T})\right).\\
  \hat{v}_{t-1}^{\text{max}}(\tilde{\bx}_{t-1}) &= \max \left\{ r(\hat{\bx}_0(\tilde{\bx}_{t-1})),\,m_t^{(k)} \right\},
  \quad
  m_{t}^{(k)} = \max_{s\ge t} r \left(\hat{\bx}_0(\bx_s^{(k)})\right).
\end{align*}
We adapted the official implementation provided \href{https://github.com/zacharyhorvitz/Fk-Diffusion-Steering}{here}.

\paragraph{TDS.} Twisted Diffusion Sampler \citep{wu2023practical} comprises of a ``twisted'' proposal which is used along with SMC to sample from the target posterior distribution. For general reward functions, the twisted proposal is the same as the one used in \Cref{eq:dps} and the final weights are obtained using \Cref{eq:smc_weights} after plugging in the twisted proposal and the value estimates:
\begin{align*}
  &q_t(\tilde{\bx}_{t-1}\mid\bx_t) = \mathcal{N}\left(\tilde{\bx}_{t-1}\,;\,\mu_{\theta}(\bx_{t},t)\,+\,\lambda\,\sigma_{t}^{2}\,\nabla_{\bx_t} r(\hat{\bx}_0(\bx_t)), \sigma_t^2\,I\right).\\
  &\hat{v}_{t-1}(\tilde{\bx}_{t-1}) = r \left(\hat{\bx}_0(\tilde{\bx}_{t-1})\right) - r \left(\hat{\bx}_0(\bx_t)\right), \quad \hat{v}_{T}(\bx_{T}) = r \left(\hat{\bx}_0(\bx_{T})\right).
\end{align*}
The official implementation is provided \href{https://github.com/blt2114/twisted_diffusion_sampler}{here}.

\paragraph{DAS.} Diffusion Alignment as Sampling \citep{kim2025test} re‑uses the twisted proposal of TDS but multiplies the reward term by a monotone tempering schedule
$0=\gamma_T\!\le\!\gamma_{T-1}\!\le\!\dots\!\le\!\gamma_0=1$ to reduce the bias from inaccurate value estimates at high noise levels. The weights are given by \Cref{eq:smc_weights} after plugging in the tempered proposal and value estimates:
\begin{align*}
  &q_t(\tilde{\bx}_{t-1}\mid\bx_t) = \mathcal{N}\left(\tilde{\bx}_{t-1}\,;\,\mu_{\theta}(\bx_{t},t)\,+\,\lambda\,\gamma_{t-1}\,\sigma_{t}^{2}\,\nabla_{\bx_t} r(\hat{\bx}_0(\bx_t)), \sigma_t^2\,I\right).\\
  &\hat{v}_{t-1}(\tilde{\bx}_{t-1}) = \gamma_{t-1}\,r \left(\hat{\bx}_0(\tilde{\bx}_{t-1})\right) - \gamma_{t}\,r \left(\hat{\bx}_0(\bx_t)\right), \quad \hat{v}_{T}(\bx_{T}) = \gamma_T\,r \left(\hat{\bx}_0(\bx_{T})\right).
\end{align*}
The official implementation is provided \href{https://github.com/krafton-ai/DAS}{here}.

\section{Additional experimental results}
\label{app:exp}




\FloatBarrier
\subsection{Class-conditional image experiments}
\label{app:exp_img}

We supplement \Cref{tab:img_results} and \Cref{fig:img_samples} with additional results and samples. 

We plot all four metrics -- FID, CMMD, average log rewards and average diversity -- across different number of function evaluations (NFEs) for the three settings considered in \Cref{sec:img_exp}. The plots show that across the three settings for most values of NFEs, \dtsa matches the target distribution more accurately compared to other methods (lowest FID and CMMD).

We also present random samples for each method and setting in \Crefrange{fig:mnist_all}{fig:cifar10_all}. We observe the same trend as noticed in \Cref{fig:img_samples} -- gradient-based guidance like DPS can be unstable leading to unnatural images, while SMC-based methods show signs of mode collapse with low average diversity and high average rewards. \dtsa balances both diversity and high rewards effectively by closely matching the true posterior distribution.


\begin{figure}[h!]
    \centering
    \includegraphics[width=\textwidth]{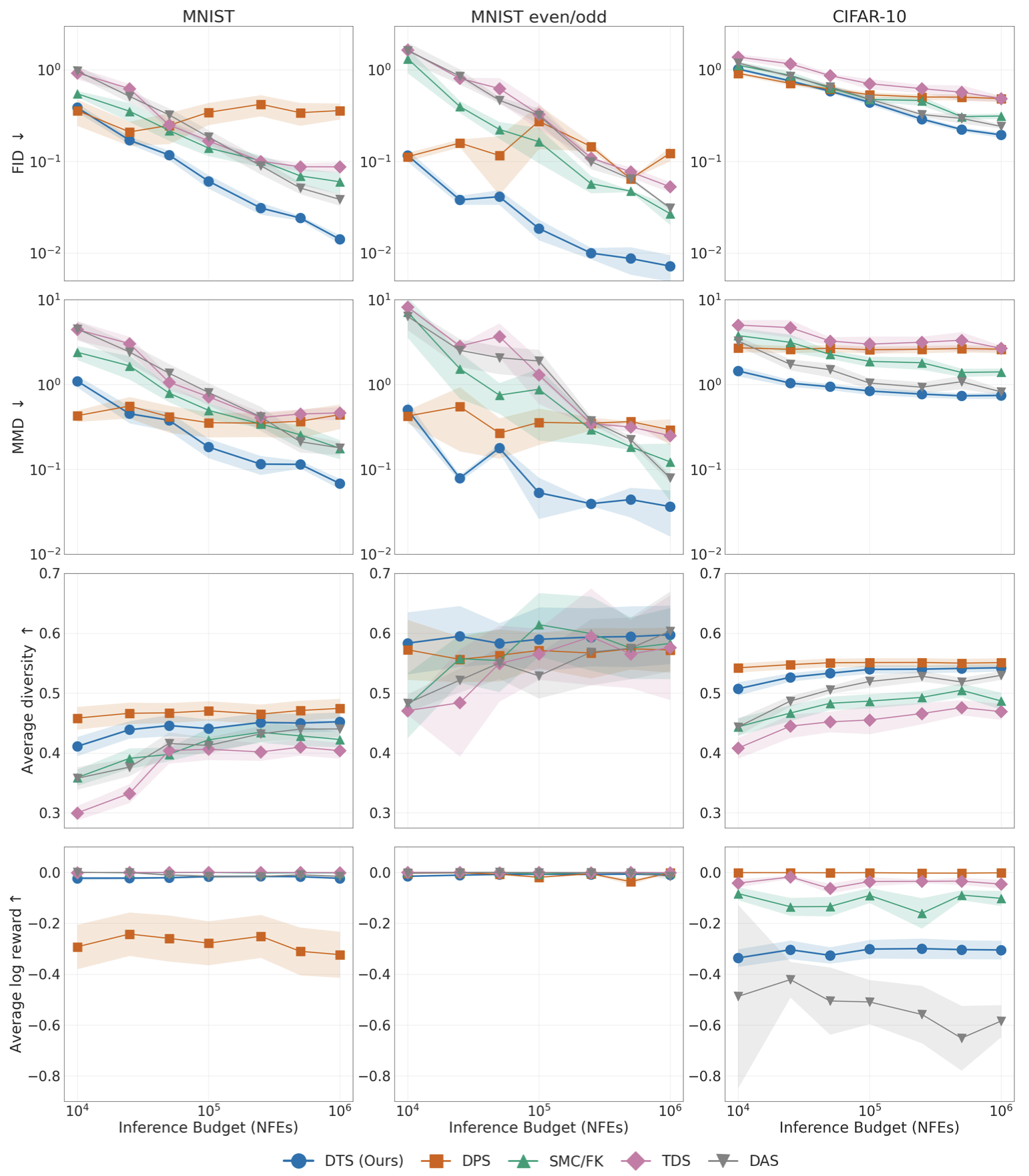}
    \label{fig:img_all_plots}
    \caption{Various distribution level metrics versus number of function evaluations for different methods on MNIST single digit generation averaged over all 10 digits (left), MNIST odd and even digit generation (center), and CIFAR-10 single class generation averaged over all 10 digits (right). All methods were evaluated with $5000$ generated samples per class. Metrics reported: FID (lower is better), CMMD (lower is better), Average log rewards (higher is better), and average diversity (higher is better).}
\end{figure}

\begin{figure}[h!]
    \centering
    \includegraphics[max size={\textwidth}{0.95\textheight}]{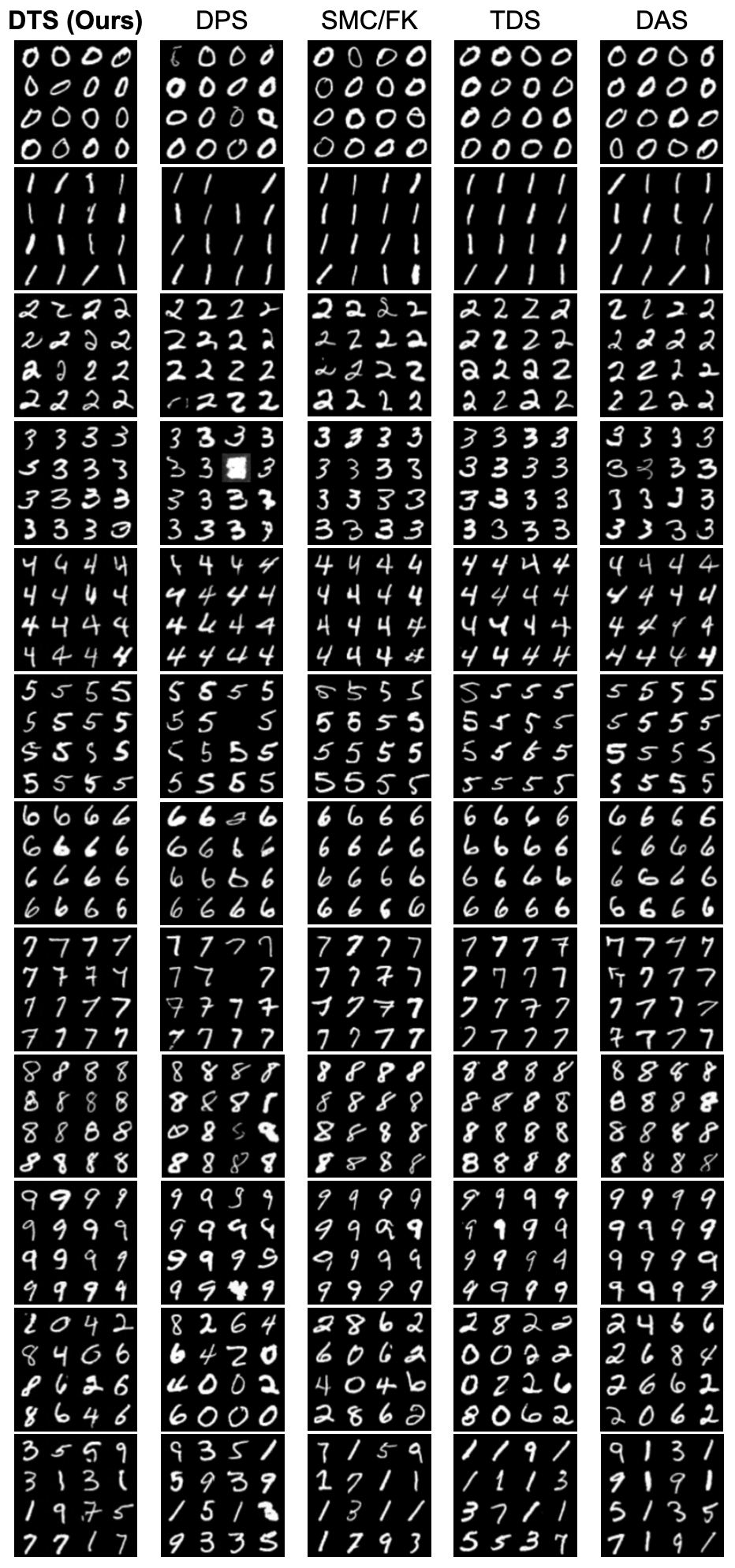}
    \caption{MNIST posterior samples generated using different methods for digits 0-9, even and odd.}
    \label{fig:mnist_all}
\end{figure}

\begin{figure}[h!]
    \centering
    \includegraphics[height=0.9\textheight]{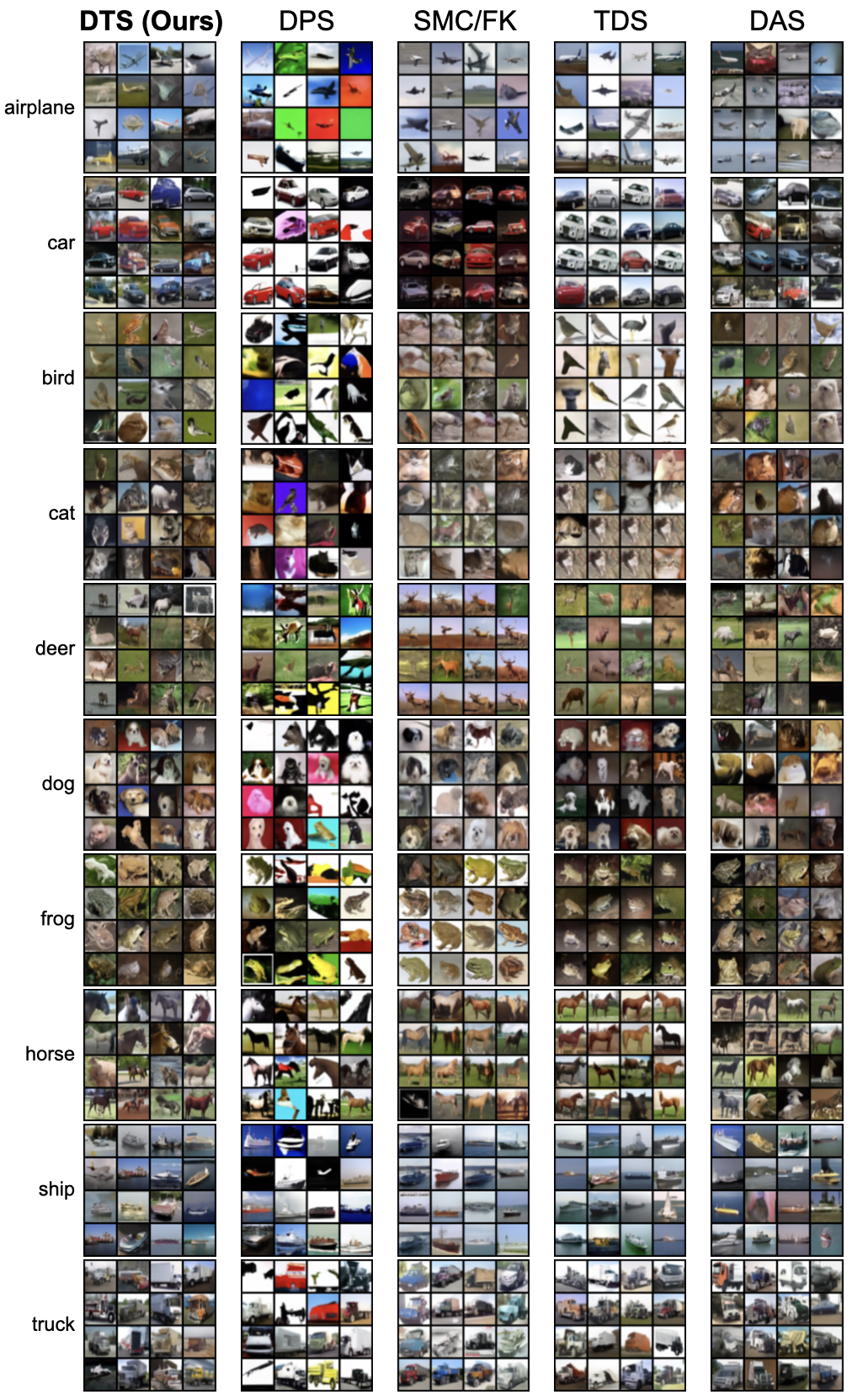}
    \caption{CIFAR-10 posterior samples generated using different methods for all classes.}
    \label{fig:cifar10_all}
\end{figure}

\clearpage
\subsection{Text-to-image examples}
\label{app:exp_sd}

We present more samples for qualitative analysis. \Cref{fig:sd_nfes} shows how samples change with increasing amount of inference-time compute, providing visual evidence for the quantitative results from \Cref{fig:sd_aesthetic}. \Crefrange{fig:sd_style}{fig:sd_count} shows text-image pairs testing different concepts such as artistic style, spatial arrangement and object count.

\begin{figure}[h!]
    \centering
    \includegraphics[width=\linewidth]{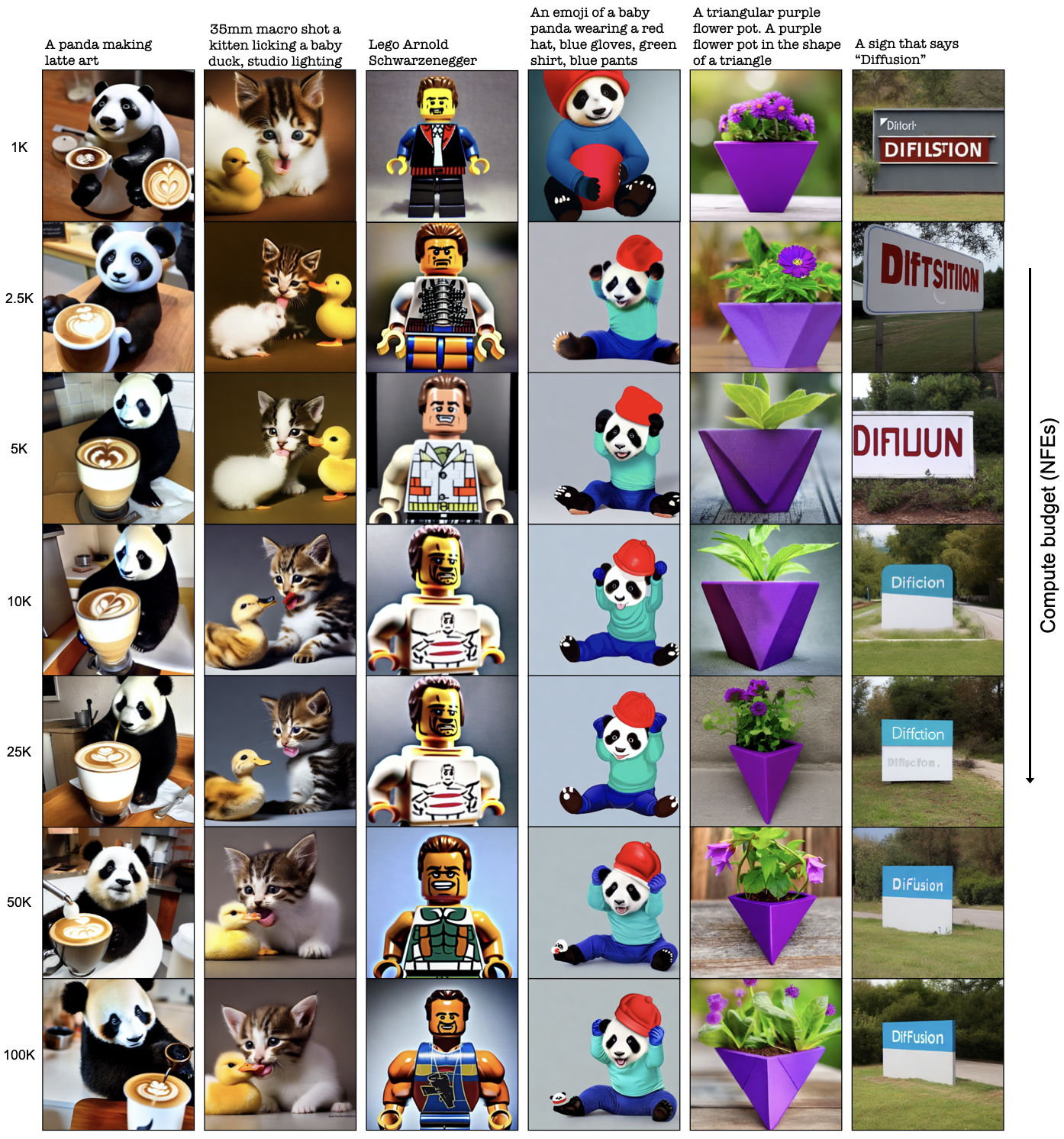}
    \caption{Text-image pairs from \Cref{fig:main_fig} with increasing amount of inference-time compute, measured in number of function evaluations (NFEs) of the diffusion model.}
    \label{fig:sd_nfes}
\end{figure}

\begin{figure}[h!]
    \centering
    \includegraphics[width=\linewidth]{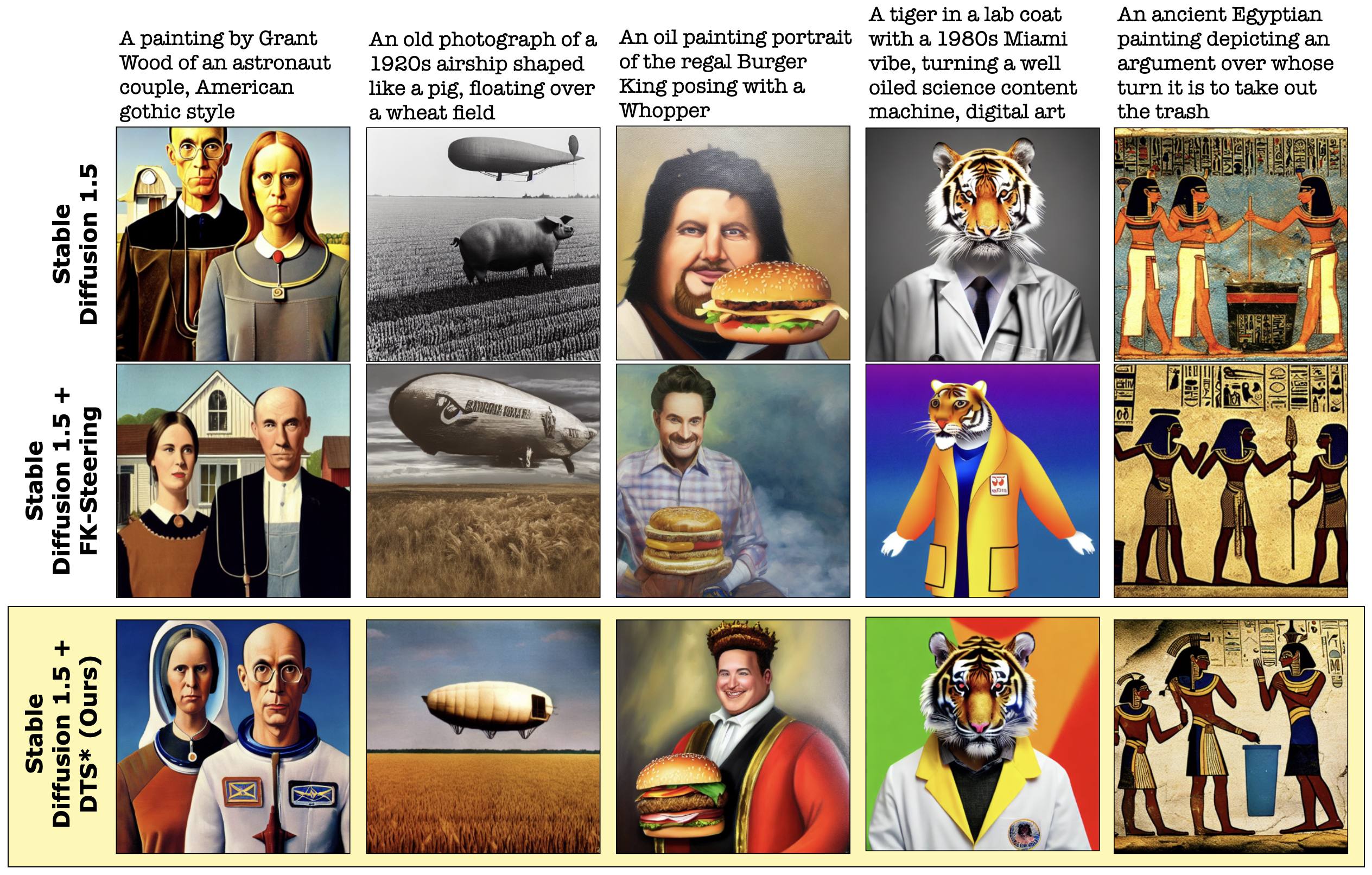}
    \caption{Sample text-image pairs using Stable Diffusion v1.5 and ImageReward as the guiding function for prompts requiring a specific artistic style. Samples are picked at random for each method and prompt.}
    \label{fig:sd_style}
\end{figure}

\begin{figure}[h!]
    \centering
    \includegraphics[width=\linewidth]{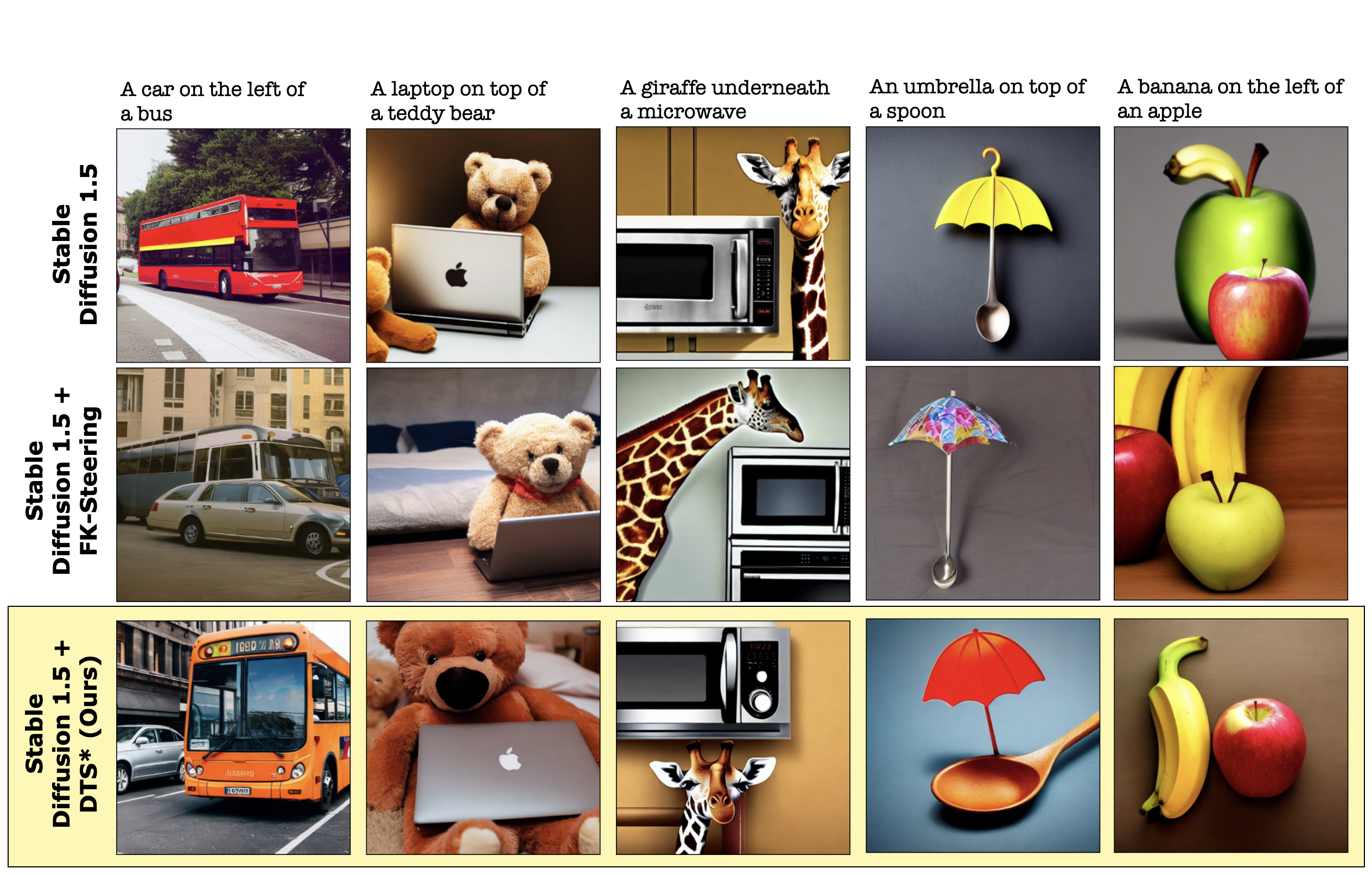}
    \caption{Sample text-image pairs using Stable Diffusion v1.5 and ImageReward as the guiding function for prompts requiring specific spatial relationships between objects. Samples are picked at random for each method and prompt.}
    \label{fig:sd_spatial}
\end{figure}

\begin{figure}[h!]
    \centering
    \includegraphics[width=\linewidth]{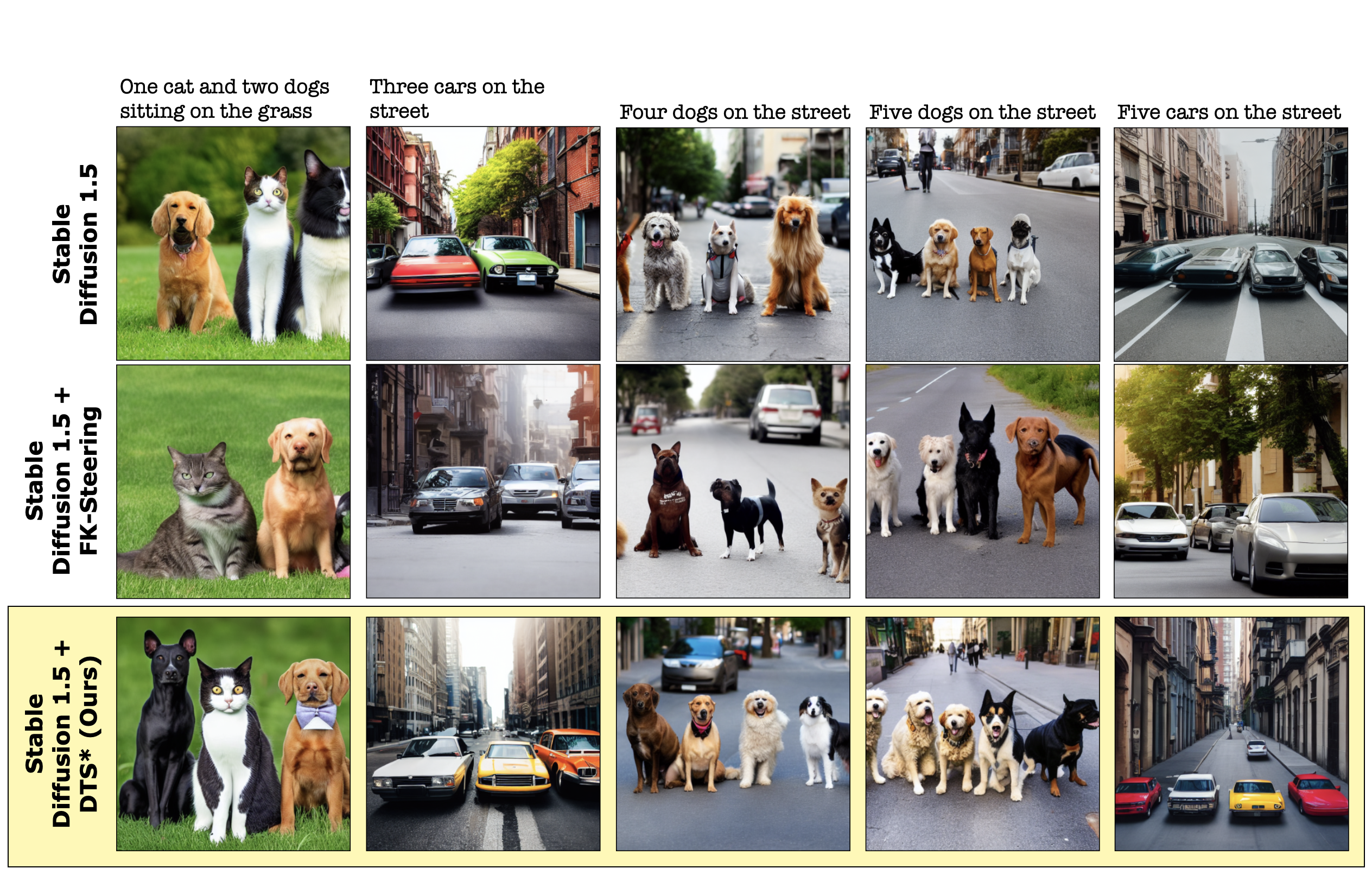}
    \caption{Sample text-image pairs using Stable Diffusion v1.5 and ImageReward as the guiding function for prompts requiring specific object counts. Samples are picked at random for each method and prompt.}
    \label{fig:sd_count}
\end{figure}

\newpage
\subsection{Text completion examples}
\label{app:exp_text}

We present additional text completions for the base MDLM model, FK-Steering and \dtse in \Cref{fig:mdlm_examples}.

\begin{figure}[h!]
    \centering
    \includegraphics[width=\linewidth]{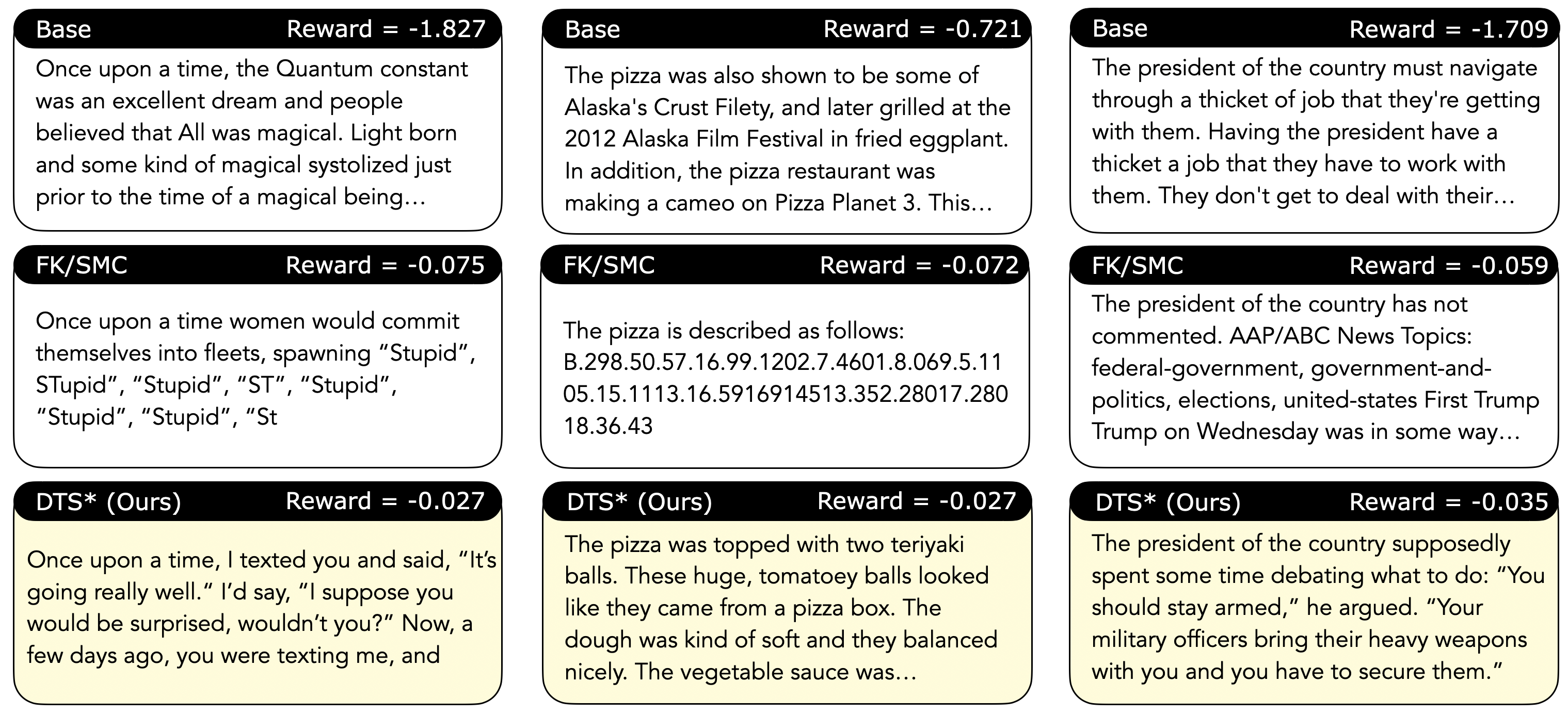}
    \caption{Sample text completions using MDLM and a CoLA classifier as reward. Samples are picked at random for each method and prompt.}
    \label{fig:mdlm_examples}
\end{figure}

\end{document}